\documentclass[11pt]{article}

\usepackage[
  letterpaper,
  left=1.5in,
  right=1.5in,
  top=1in,
  bottom=1in,
  footskip=0.6in
]{geometry}
\usepackage[english]{babel}
\usepackage[T1]{fontenc}
\usepackage{newtxtext}
\usepackage{microtype}
\usepackage{setspace}
\usepackage{titlesec}
\usepackage{etoolbox}

\usepackage{amsmath,amssymb,amsfonts,amsthm,mathtools}
\usepackage{newtxmath}
\usepackage{bbm}

\usepackage{graphicx}
\usepackage{caption}
\usepackage{subcaption}
\usepackage{booktabs}
\usepackage{multirow}
\usepackage{float}
\usepackage{algorithm}
\usepackage{algpseudocode}

\usepackage{comment}
\usepackage{xcolor}
\usepackage{xspace}
\usepackage{natbib}
\usepackage{authblk}
\usepackage[
  colorlinks=false,
  linkbordercolor={1 0 0},
  citebordercolor={0 1 0},
  urlbordercolor={0 1 1},
  pdfborder={0 0 1}
]{hyperref}
\usepackage{cleveref}

\AtBeginDocument{%
  \setlength{\abovedisplayskip}{6pt plus 2pt minus 1pt}%
  \setlength{\belowdisplayskip}{6pt plus 2pt minus 1pt}%
  \setlength{\abovedisplayshortskip}{3pt plus 1pt}%
  \setlength{\belowdisplayshortskip}{4pt plus 1pt minus 1pt}%
  \allowdisplaybreaks[1]%
}

\titleformat{\section}
  {\large\bfseries}{\thesection.}{0.65em}{}
\titleformat{\subsection}
  {\normalsize\bfseries}{\thesubsection.}{0.55em}{}
\titleformat{\subsubsection}
  {\normalsize\itshape}{\thesubsubsection.}{0.55em}{}
\titlespacing*{\section}{0pt}{2.4ex plus 0.6ex minus 0.2ex}
  {1.0ex plus 0.2ex}
\titlespacing*{\subsection}{0pt}{1.9ex plus 0.4ex minus 0.2ex}
  {0.7ex plus 0.2ex}
\titlespacing*{\subsubsection}{0pt}{1.5ex plus 0.3ex minus 0.2ex}
  {0.5ex plus 0.1ex}

\newtheoremstyle{statplain}
  {6pt}{6pt}{\itshape}{}
  {\bfseries}{.}{0.5em}{}
\newtheoremstyle{statdefinition}
  {6pt}{6pt}{\normalfont}{}
  {\bfseries}{.}{0.5em}{}

\theoremstyle{statplain}
\newtheorem{theorem}{Theorem}[section]
\newtheorem{lemma}[theorem]{Lemma}

\newtheorem{corollary}[theorem]{Corollary}
\newtheorem{claim}[theorem]{Claim}

\theoremstyle{statdefinition}
\newtheorem{definition}[theorem]{Definition}

\newtheorem{remark}[theorem]{Remark}

\newenvironment{proofsketch}
  {\begin{proof}[Proof sketch]}
  {\end{proof}}

\DeclareMathOperator{\Var}{Var}

\newcommand{\mc}{\mathrm{mc}}
\newcommand{\tp}{\mathrm{tp}}

\newcommand{\stat}{\mathrm{stat}}
\newcommand{\sumin}{\sum_{i=1}^n}
\newcommand{\ra}{\rightarrow}
\newcommand{\la}{\leftarrow}

\newcommand{\Z}{\mathbb{Z}}

\newcommand{\E}{\mathbb{E}}
\newcommand{\PP}{\mathbb{P}}

\newcommand{\I}{\mathcal{I}}
\newcommand{\J}{\mathcal{J}}

\newcommand{\X}{\mathcal{X}}

\renewcommand{\Affilfont}{\normalfont}
\AtBeginEnvironment{thebibliography}{\small\setstretch{1}}
\providecommand{\xuelinbibstyle}{plainnat}

\definecolor{myGreen}{RGB}{80,180,0}
\definecolor{myGold}{rgb}{0.75,0.6,0.12}
\definecolor{myBlue}{rgb}{0.12,0.45,0.75}

\newcommand{\bigO}{\mathcal{O}}
\newcommand{\interval}{I} 
\newcommand{\intervalset}{\I} 
\newcommand{\numInterval}{J} 
\newcommand{\iteratorInterval}{j} 

\newcommand{\numClass}{G} 
\newcommand{\iteratorClass}{g} 

\newcommand{\numData}{n} 

\newcommand{\failureProb}{\delta} 

\newcommand{\discreteNum}{K} 

\newcommand{\estModel}{p} 

\newcommand{\titleName}{Prediction-Interval-Conditional Prediction Interval\xspace}
\newcommand{\titleShort}{PICPI\xspace}
\newcommand{\titleNames}{Prediction-Interval-Conditional Prediction Intervals\xspace}
\newcommand{\titleShorts}{PICPIs\xspace}

\newif\ifjrssb
\jrssbfalse 
\usepackage{rotating}

\title{PICPIs: Prediction-Interval-Conditional Prediction Intervals}

\author[{}]{%
  Xuelin Yang\textsuperscript{1,}\thanks{Equal contribution. Imbens' work is supported by the Office of Naval Research under Grant N00014-17-1-2131; Jordan’s work is supported by European Union under Grant ERC-2022-SYG-OCEAN-101071601. Email correspondence: \texttt{xuelin@berkeley.edu}.  Code: \texttt{https://github.com/xyang23/picpi}.}\qquad
  Baihe Huang\textsuperscript{1,}\protect\footnotemark[1]\qquad
  Yilong Hou\textsuperscript{1,}\protect\footnotemark[1]\authorcr
  Guido Imbens\textsuperscript{2} \qquad
  Michael I. Jordan\textsuperscript{1,}\textsuperscript{3}\\
  {\Affilfont
    \textsuperscript{1}University of California, Berkeley\\
    \textsuperscript{2}Stanford University\\
    \textsuperscript{3}Inria Paris}
}
\date{}
\begin{document}

\maketitle
\begin{abstract}
  A classical question in statistics is which observable quantities to
condition on when drawing inferences about unobservable targets.
For conformal prediction in nonparametric uncertainty quantification, standard marginal validity offers limited resolution at the prediction values on which decisions are based, and fully conditional guarantees with respect to the covariates are provably unattainable. 
We address this gap by introducing a prediction-based conditioning framework that we refer to as \titleNames (\titleShorts). Formally,  a \titleShort is an interval $I$ satisfying a self-consistency condition:
$$ \E [ Y\mid p(X) \in I] \in I,$$
for predictive model $p$, contextual covariate $X$, and outcome $Y$.
Thus, an interval simultaneously defines a stratum of prediction values and certifies that the mean outcome in that stratum lies in the same interval.
This self-consistency condition yields data-adaptive strata without altering the original prediction.
Such intervals can be constructed using practical algorithms.  
Under regularity of the prediction distribution, the constructed intervals cover all but an arbitrarily small fraction of prediction values and have widths that decrease at rate $n^{-1/3}$, up to logarithmic factors and the prediction error.
Moreover, identifying these locally calibrated intervals can, in turn, inform downstream decision-making.
We derive inference procedures for \titleShorts in probabilistic prediction and multi-class classification, accompanied by theoretical guarantees and empirical results.


\end{abstract}

\newpage
\section{Introduction}

A classical problem in statistical inference is that of
\emph{conditioning}---on what observable quantities should one condition in order to make inferences about unobservable quantities 
\ifjrssb
~\citep{fisher1934statistical,cox1958some, reid1995}?
\else
~\citep{fisher1934statistical,cox1958some,  cox1988some, reid1995}?
\fi
This problem has received renewed attention in recent work on nonparametric uncertainty quantification,
specifically the area of \emph{conformal prediction}, which aims to produce confidence intervals for nonparametric prediction~\citep{vovk2005algorithmic,shafer2008tutorial}. Consider for example the following formal model for prediction. Let $X\in\mathcal{X}$ denote a vector of covariates, let $Y\in \mathcal{Y}$ denote a bounded outcome, and let $p:\mathcal{X}\to \mathcal{Y}$ be a fitted model whose output $p(X)$ is intended to approximate the conditional mean $p^*(X):=\E[Y\mid X]$. Assuming exchangeability for the observable data, $\{(X_i, Y_i)\}_{i=1}^n$, but making no other structural or probabilistic assumptions, conformal prediction yields distribution-free prediction sets that come with a marginal guarantee of the form:
\begin{align}\label{eq:intro_conformal_marginal}
\PP\Big(Y\in C(X)\Big)\ge 1-\alpha,
\end{align}
where $C(X)$ is a prediction set constructed from $p(X)$, and $\alpha \in (0,1)$ is a user-specified failure probability. As has been emphasized in recent years, this marginal guarantee is unconditional with respect to the covariate $X$, and thus does not provide direct assessment of uncertainty about a specific prediction; moreover, fully conditional coverage with respect to $X$ is generally unattainable without strong distributional assumptions.  More formally,
\citet{lei2014distribution} and \citet{barber2021limits} demonstrate that any prediction interval $C(X)$ satisfying the strict conditional validity guarantee
\begin{align}\label{eq:conditional_conformal}
    \PP\Big(Y \in C(x)\mid X=x\Big) \geq 1-\alpha \quad \text{for almost all } x,
\end{align}
must be vacuous.  A significant effort has been made to temper the message in this result by enforcing conditional validity via coarsening of the covariates~\citep{ding2023class, martinezgil2024groups, gibbs2025conformal,duchi2025few}. 

We present an alternative perspective on the problem of obtaining conditional estimates of uncertainty. Rather than conditioning on the covariate $X$, we suggest conditioning on the prediction $p(X)$. The idea of prediction-based conditioning is not new---indeed, it is the foundation of the literature on probabilistic calibration~\citep{dawid1982wellcalibrated,foster1998asymptotic}.  Recall that a calibrated model is defined as one that satisfies the following conditional criterion:
\begin{align}\label{eq:intro_cali}
    \E\Big[Y\mid p(X)=v\Big] = v.
\end{align} 
That is, among all instances where the predicted outcome is approximately $v$, the expected realized outcome should be approximately $v$. 

Our approach to conditional uncertainty quantification is built on a similar ``self-consistency'' property. Specifically, we ask that the expected realized outcome, conditioning on the model's prediction falling within a specific interval $I$, also lies within $I$. 
Formally, we have:
\begin{align}\label{eq:intro_picpi}
\E \Big[Y \mid p(X)\in I\Big]\in I.
\end{align}
Effectively, we are using the model's predictions to stratify risk. We refer to this approach as \emph{\titleNames (\titleShorts)}.  
As illustrated in Figure~\ref{fig:teaser}, each \titleShort serves a dual role: it simultaneously defines the stratum of predictions being evaluated and provides a calibrated uncertainty statement for that specific stratum. 
In this way, uncertainty quantification is aligned with prediction, producing a summary of model reliability across the prediction distribution. 
The \titleShort condition can be transferred into a conformal-styled guarantee (in Theorem~\ref{thm:tp}).
In contrast to the intervals defined in Eq.~\eqref{eq:conditional_conformal}, we show (in Theorem~\ref{thm:consistency-prediction}) that the width of the \titleShorts converges to zero under suitable assumptions, providing increasingly precise characterizations of risk. 


\begin{figure}[tb]
    \centering
    \vspace{2mm}
    \includegraphics[width=0.45\linewidth]{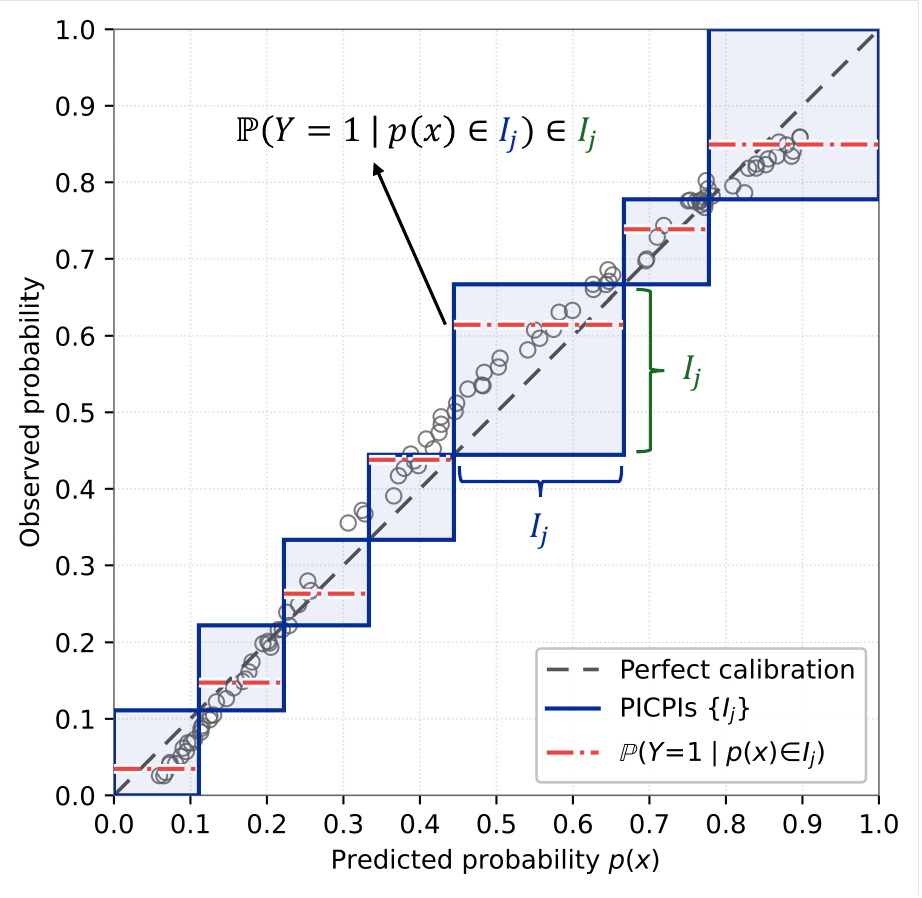}
    \caption{Illustration of the \titleShort methodology under binary outcome. 
    Each blue interval $I_j$ is both a stratum of predicted probabilities and a certificate that its empirical outcome frequency, shown in red, lies in $I_j$.
    The dashed diagonal line corresponds to perfect calibration. This example is generated under model misspecification (quadratic ground truth with a linear fitted model).}
    \label{fig:teaser}
\end{figure}

We make the following two distinctions in comparison with standard methodology in calibration.
First, calibration is usually assessed via binning. 
Binning schemes are typically based on \emph{fixed, external} partitions of the prediction space, with overall calibration error computed as an aggregation of the errors within each bin.
While useful for general model assessment, fixed bins need not align with the decision boundaries that arise when optimizing a downstream risk function. 
The proposed \titleShorts provide  local, data-adaptive diagnostics for  calibration.
Rather than forcing the model to conform to pre-defined bins, we characterize regions of the prediction space where the model is  ``self-consistent'' and thus reliable. 

Second, calibration error computed from binning schemes is often used to guide \emph{post-hoc} recalibration procedures such as Platt scaling, isotonic regression, or self-calibration \citep{platt1999probabilistic,  van2024self,berta2024classifier}. 
Recalibration, however, can be undesirable \citep{guo2017calibration, pleiss2017fairness, carriero2025harms}. 
With \titleShorts, a practitioner can instead identify regions of the prediction space that are directly actionable. We further demonstrate how the knowledge of these locally calibrated intervals can inform downstream decision-making, including probabilistic prediction and classification.
An extended discussion of these issues is provided in Section~\ref{sec:related_work}.

Our specific contributions are as follows: 
\begin{enumerate}
    \item We propose \titleShorts as a novel approach to conditional uncertainty quantification, we design algorithms to construct \titleShorts, and we provide theoretical characterizations of these intervals. In particular, the \titleShorts constructed via our algorithm cover all but a user-specified prediction mass and have widths controlled according to a $\Tilde{\bigO}(n^{-1/3})$ term plus the prediction error (Theorem~\ref{thm:consistency-prediction}). Consequently, these widths vanish when the base predictor is consistent.
    \item We provide inference procedures for probabilistic prediction and multi-class classification, and present both theoretical guarantees and empirical results. Specifically, we show that probabilistic prediction using \titleShorts has marginal miscoverage that is bounded by the within-stratum variation of $p^*(X)$ together with the prediction mass not covered by the \titleShorts (Theorem~\ref{thm:tp}).
Classification using \titleShorts has worst-case class-conditional coverage certificates (Theorem~\ref{thm:mc_error}), a set-size optimality result given a prescribed coverage threshold (Theorem~\ref{thm:mc_optimality}), which allows data re-use when constructing intervals (Theorem~\ref{thm:mc_same_sample}).
\end{enumerate}

The remainder of the paper is organized as follows: Section~\ref{sec:def-alg} introduces the \titleShort definition and presents its algorithmic instantiation.  
Section~\ref{sec:properties} derives theoretical properties of \titleShorts. 
Section~\ref{sec:tasks} presents an inference framework using \titleShorts for probabilistic prediction and classification. 
Section~\ref{sec:experiment} discusses the empirical performance of the \titleShort approach in comparison with several standard interval-based baselines.
Section~\ref{sec:related_work} presents related work and we present our conclusions in Section~\ref{sec:conclusions}.

\section{Definition, Algorithmic Construction, and Validity}\label{sec:def-alg}

Without loss of generality, we focus the exposition on binary outcomes, which can be viewed as a special case where $\E [Y\mid X]=\PP (Y=1\mid X)$. Let $Y\in\{0,1\}$, let $X\in\X$, and let $p:\X\to[0,1]$ be a fitted
model for $p^*(x):=\PP(Y=1\mid X=x)$. The model may be trained on a separate sample or supplied as an off-the-shelf predictor. A PICPI is an interval in the prediction space that both groups similar predictions and contains the group's underlying outcome probability.

\begin{definition}[\titleName (\titleShort)]\label{def:picpi}
A collection $\intervalset_p=\{\interval_j\}_{j=1}^{J}$ is called a
collection of \emph{\titleNames (\titleShorts)} with respect to $p$ if,
for every $j\in[J]$ with $\PP(p(X)\in\interval_j)>0$,
\begin{align}
\PP(Y=1\mid p(X)\in\interval_j)\in\interval_j.
\label{eq:picpi-definition}
\end{align}
\end{definition}

We note that the \titleShort notion extends directly to regression with a bounded real-valued outcome, as we will formalize in Remark~\ref{rmk:extension_regression}. 


\begin{algorithm}[bth] 
   \caption{\titleShort Construction (population mode with theoretical validity)} 
   \label{alg:calibration_population}
   \small
\begin{algorithmic}[1]
\Procedure{\textsc{Calibration}}{data points $\{(x_i,y_i)\}_{i=1}^\numData,$ predictor $ p$, failure probability $\delta$}

\Comment{Obtain a set of \titleShorts}
\State Let $\{0,1/\discreteNum,2/\discreteNum,\cdots,1\}$ be a discretization of $[0,1]$.
\State Initialize \titleShorts $\intervalset_p \la \emptyset$ 
\For {$0\leq i<j \leq\discreteNum$} 
\State Let $[a,b]\la [i/\discreteNum, j/\discreteNum]$ \Comment{Evaluate the current interval}
\If{$[a,b]$ satisfies $\sum_{i=1}^n \mathbbm{1}(p(x_i) \in [a,b]) > 0$ and 
\begin{align*}
    \frac{\sum_{i=1}^n y_i \cdot \mathbbm{1}(p(x_i) \in [a,b])}{\sum_{i=1}^n \mathbbm{1}(p(x_i) \in [a,b])} \in \left[a + 2\sqrt{\frac{\log (\discreteNum^2/\delta)}{\sum_{i=1}^n \mathbbm{1}(p(x_i) \in [a,b])}}  ,b - 2\sqrt{\frac{\log (\discreteNum^2/\delta)}{\sum_{i=1}^n \mathbbm{1}(p(x_i) \in [a,b])}}\right]
\end{align*}} \label{line:empirical-interval}

\State $\intervalset_p \la \intervalset_p \cup \{[a, b]\}$ 

\EndIf
\EndFor
\State {\bf Return} \titleShorts $\intervalset_p$
\EndProcedure
\\
\Procedure{\textsc{Inference}}{covariate input $x,$ predictor $ p$, failure probability $\delta$}
\Comment{Find the shortest \titleShort that contains $p(x)$}
\State $\intervalset_p \leftarrow \textsc{Calibration}\left(\{(x_i,y_i)\}_{i=1}^n, p, \delta \right)$
\State Let $[a,b]$ be the shortest interval in $\intervalset_p$ that contains $p(x)$
\State {\bf Return} $[a,b]$
\EndProcedure
\end{algorithmic}
\end{algorithm}

We now present an algorithm to construct \titleShorts with an accompanying validity statement. 
The population-level algorithm, presented in pseudocode as Algorithm~\ref{alg:calibration_population}, comprises two procedures.  \textsc{Calibration} returns
a collection of intervals that we will show to satisfy the \titleShort condition (Theorem~\ref{thm:algorithm_population}).  Once this collection has been computed,
\textsc{Inference} reports the shortest returned interval containing a new
prediction $p(x)$. If a disjoint summary is preferred, we select a disjoint subset with maximum covered length through the \textsc{DisjointCalibration} procedure provided in Supplementary Material~\ref{app:alg}.

The theoretical \titleShort validity of the produced intervals---that is, whether they satisfy Definition~\ref{def:picpi} defined on the population---follows from a concentration argument. In particular, for a candidate
interval $I=[a,b]$, the \titleShort condition asks whether the unknown
conditional mean
\(
\PP(Y=1\mid p(X)\in I)
\)
lies between $a$ and $b$.  We estimate this quantity by the average label
among calibration observations whose predictions fall in $I$. Because this
average could be noisy, it is not enough for it merely to lie in $I$: we require it
to be separated from both endpoints by a concentration margin (Line~\ref{line:empirical-interval}).  The margin is
larger for intervals containing fewer calibration observations, precisely
where the empirical average is less reliable.
This ensures that every interval admitted by the algorithm satisfies the population-level \titleShort property.
We formalize this statement of validity in the following theorem:
\begin{theorem}\label{thm:algorithm_population}
Suppose $\{(x_i,y_i)\}_{i=1}^{\numData}$ are drawn i.i.d. from the
distribution of $(X,Y)$ and $p$ is fixed independently of these observations.
Then, with probability at least $1-\failureProb$ over the calibration sample,
every interval returned by \textsc{Calibration} in
Algorithm~\ref{alg:calibration_population} is a \titleShort with respect to
$p$.

\end{theorem}

\begin{proofsketch}
We use Hoeffding's inequality on $\PP(Y=1\mid p(X)\in I)$ and a union bound over $\discreteNum^2$ candidate intervals. For the full proof, see  Supplementary Material~\ref{proof_thm:algorithm_population}.
\end{proofsketch}

Theorem~\ref{thm:algorithm_population} guarantees the PICPI property under the true underlying distribution. An empirical version of the algorithm that we will discuss in a later section (Algorithm~\ref{alg:calibration_empirical}, Section~\ref{sec:alg_emp}) produces PICPIs with respect to the empirical distribution of the calibration sample. 
Conceptually, their relationship is analogous to the distinction between a population CDF and an empirical CDF: while Theorem~\ref{thm:algorithm_population} enforces the PICPI property with respect to the  underlying population measure, Algorithm~\ref{alg:calibration_empirical} satisfies it under the plug-in sample measure.

\clearpage

\section{Properties of \titleShorts}
\label{sec:properties}

In this section we show that, when a predictor has error $\epsilon$ and its predicted probabilities satisfy a $\lambda$-regularity condition, most predictions will be enclosed by \titleShorts. As the calibration sample size grows, these intervals shrink while their union continues to cover most of the prediction mass; the rate depends on $\epsilon$, $\lambda$, and the calibration sample size.
Our result requires a regularity condition ensuring that the predicted probabilities are sufficiently spread across $[0,1]$.

\begin{definition}[$\lambda$-regularity]
A random variable $Z$ over $[0,1]$ is $\lambda$-\emph{regular} if $\PP(Z \in A) \geq \lambda \cdot \mu_B(A)$ for any measurable set $A$, where $\mu_B$ is the Lebesgue measure.
\end{definition}


\begin{theorem}
[Finite-sample width guarantee for \titleShorts]
\label{thm:consistency-prediction}
Consider $\{(x_i,y_i)\}_{i=1}^{m+n}$  i.i.d.\ data points such that 
$n \geq 2^{-1/2} \lambda^{-1} \log(3n/\delta) $.
Use $m$ data points to construct the predictor $\estModel: \X \to [0,1]$ and $n$ data points to produce \titleShorts $\intervalset_p$  via Algorithm~\ref{alg:calibration_population}. 
Suppose $X \sim \mu$, $\estModel(X)$ is $\lambda$-regular, and
\begin{align*}
    \E_\mu\Big[\Big| \estModel (X) - \PP(Y = 1 \mid X)\Big|\Big] \leq \epsilon, \qquad \text{for some  $\epsilon>0$.}
\end{align*}
For any $\delta,\delta_1 > 0$, with probability at least $1-\delta$ over the randomness of $\{(x_i,y_i)\}_{i=1}^n$,
\begin{align}
    \PP_{x \sim \mu}\Big(\exists [a,b] \in \intervalset_p ~s.t.~ \underbrace{p(x) \in [a,b]}_{\text{enclose prediction}} ~\text{and}~ \underbrace{|b-a| \leq C(n,\epsilon,\delta,\delta_1)}_{\text{bounded width}}  \Big) \geq 1-\delta_1,
\end{align}
where 
\begin{align}\label{eq:width_function}
    C(n,\epsilon,\delta,\delta_1) = c_0 \log(1/\delta_1) \cdot \Big(\Big(\underbrace{\frac{ \log(n/\delta)}{\lambda n}\Big)^{1/3}}_{\text{calibration error}}  + \underbrace{\frac{\epsilon}{\delta_1}}_{\text{prediction error}}\Big)
\end{align}
for a universal constant $c_0$.

\end{theorem}

\begin{proofsketch}
    First, we construct an initial collection of intervals of desired length and coverage by applying Lemma~\ref{lem:dist-regularity} to the distribution of $p(X)$. That is, there exist intervals $\{[a_j,b_j]\}$ where the conditional mean prediction $\E[p(X) \mid p(X)\in [a_j, b_j]]$ is bounded away from its own endpoints. Then we transfer this property 
    to the conditional mean outcome $\E[Y \mid p(X)\in [a_j, b_j]]$ (by setting $U=p^*(X)$ and $V=p(X)$ in Lemma~\ref{lem:conditional-dist-shift}). Upon discarding a small amount of mass, the retained intervals satisfy the desired \titleShort property. We then show that these intervals can be produced by Algorithm~\ref{alg:calibration_population}. Define $A_1$ to be the event that every interval contains sufficiently many samples, and $A_2$ the event that the empirical label average is close to the conditional mean outcome. Under $A_1\cap A_2$, the empirical estimate is sufficiently far from both endpoints for the algorithm to accept. By Lemma~\ref{lem:sample-size-interval}, 
    $A_1\cap A_2$ occurs with probability $1-\delta$. For the full proof see
Supplementary Material~\ref{proof:thm:consistency-prediction}.
\end{proofsketch}
One can plug in parametric rates for the prediction error $\epsilon$ in Eq.~\eqref{eq:width_function}. We note that Eq.~\eqref{eq:width_function} serves as an upper bound, while developing adaptive bounds dependent on additional model properties or parametric structure is an open direction.


\begin{figure}[ht]
    \centering
    \vspace{1mm}
    \begin{subfigure}[b]{0.52\textwidth}
        \centering
        \includegraphics[width=\textwidth]{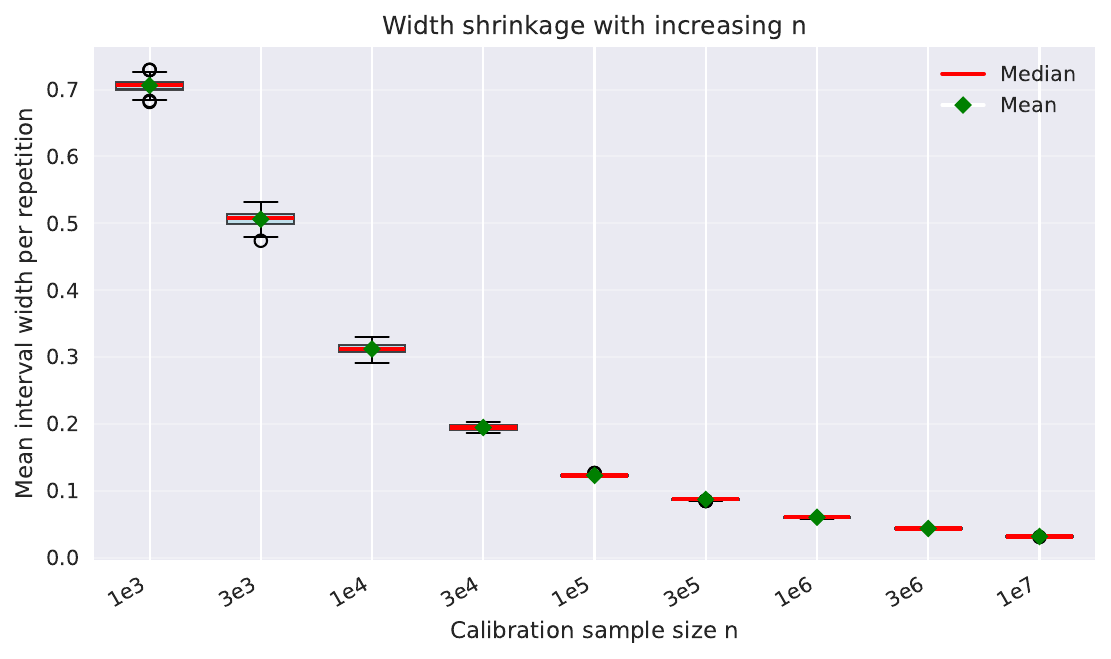} 
  
        \label{fig:thm52_box_rate_a}
    \end{subfigure}
    \hfill 
    \begin{subfigure}[b]{0.45\textwidth}
        \centering
        \includegraphics[width=\textwidth]{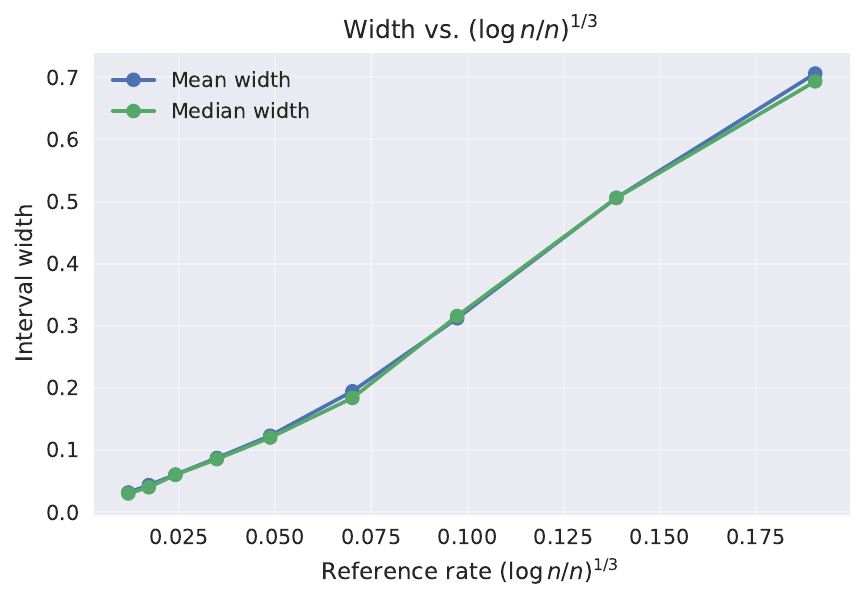} 
        
        \label{fig:thm52_box_rate_2}
    \end{subfigure}
    
    \caption{Demonstration of width shrinkage. For implementation details see Supplementary Material~\ref{app:verification-consistency-prediction-setup}.}
    \label{fig:thm52_box_rate}
\end{figure}

\subsection{Supporting lemmas}

As shown in the proof sketch, Theorem~\ref{thm:consistency-prediction} relies on three supporting lemmas. They are presented here to illustrate the key arguments needed to establish the theorem.

\begin{lemma}[Sample size in each interval]\label{lem:sample-size-interval}
Let $\nu$ be a $\lambda$-regular distribution over $[0,1]$ and let $[a_\iteratorInterval,b_\iteratorInterval], \iteratorInterval \in [\numInterval]$ be intervals such that $|b_\iteratorInterval - a_\iteratorInterval| \geq c$ for all $\iteratorInterval \in [\numInterval]$. For any $n \in Z_+$, let $z_1,\dots,z_n$ be i.i.d. samples from $\nu$, then with probability at least 
$1-J\exp\Big(-\lambda cn/8\Big)$
, it holds for all $j \in [J]$ that
\begin{align*}
    \sumin \mathbbm{1}(z_i \in [a_j,b_j]) \geq \lambda cn/2.
\end{align*}

\end{lemma}

\begin{proofsketch}
Write the sample count in each interval as a sum of Bernoulli variables and lower-bound their expectations using the $\lambda$-regularity condition. The result follows from a lower-tail Chernoff bound. The full proof can be found in 
Supplementary Material~\ref{proof_lem:sample-size-interval}.
\end{proofsketch}

\begin{lemma}[Partition of bounded conditional mean]\label{lem:dist-regularity}
For any $\epsilon, \delta > 0$, and any distribution $\nu$ over $[0,1]$, there exist disjoint intervals $[a_i,b_i], i \in [L]$ such that:
\begin{enumerate}
    \item \emph{(Well-centered conditional means)}  
    The conditional mean of $Z$ on $[a_i,b_i]$ lies away from the interval boundaries by at least $\epsilon$ (except for the endpoints $0$ or $1$). That is, for any $i \in [L]$, 
    \begin{align}\label{eq:regular-condition}
        a_i + \epsilon \cdot \mathbbm{1}(a_i \neq 0) \leq \E_{Z \sim \nu}[Z \mid Z \in [a_i,b_i]] \leq b_i - \epsilon \cdot \mathbbm{1}(b_i \neq 1).
    \end{align}
    \item 
    \emph{(Coverage guarantee)}  
    The union of these intervals covers all but a $\delta$-fraction of the probability mass. That is, \begin{align}\label{eq:regular-condition-coverage}
        \PP_{Z \sim \nu}\Big(Z \notin \cup_{i \in [L]} [a_i,b_i]\Big) \leq \delta.
    \end{align}
    \item \emph{(Bounded lengths)}  
    Each interval has a bounded length. That is, for any $i \in [L]$,
    \begin{align}\label{eq:regular-condition-length}
        4\epsilon \leq \Big|b_i - a_i\Big| \leq 56 \epsilon \cdot \log(3/\delta).
    \end{align} 
\end{enumerate}
\end{lemma}




\begin{proof}[Proof Sketch]
We provide a constructive proof. Discretize $[0,1]$ into $K = \lfloor1/(4\epsilon)\rfloor$ subintervals of length $4\epsilon$ and form a directed graph $G=(V,E)$ on $V = [K]$. We draw a directed edge $i \to j$ ($|i-j|=1$) if either the conditional mean of $Z$ in bin $i$ $\E[Z \mid Z \in B_i]$ is closer to the neighboring endpoint on the side of bin $j$, or the probability mass of $Z$ in bin $i$ is not substantially greater than bin $j$ (i.e., $\PP(Z\in B_i)\le 3\PP(Z\in B_j)$). By construction, every vertex has at least one outgoing edge, and there exists  bidirectional edges. 
Group consecutive bidirectional edges into maximal blocks. Partition each block into groups of two vertices. Notice that the edges between blocks split into a leftward chain followed by a rightward chain. Attach them to the neighboring blocks, truncated at a radius of $\lceil\log(3/\delta)\rceil$. This would give the desired intervals.
We then show these intervals satisfy the stated properties. 
The boundedness of widths follows from truncation of the chain. 
The conditional mean property comes from the edge construction and total expectation. Finally, outside these intervals, the absence of reverse edges implies exponential mass decay to bound uncovered mass by $\delta$. For the full proof see
Supplementary Material~\ref{proof_lem:dist-regularity}.
\end{proof}

\begin{lemma}[Bounds under conditional distribution shift]\label{lem:conditional-dist-shift}
Let $U,V$ be random variables jointly over $[0,1]$, such that $\E[|U-V|] \leq \epsilon > 0$. Then for any $\delta > 0$, integer $L$ and intervals $[a_i,b_i], i = 1,\dots, L$ with disjoint interiors such that $\PP\Big(V \in \cup_{j \in [L]} [a_j,b_j]\Big) > 1- \delta$, 
there exists a subset $\{i_1,\dots,i_\numInterval\}$ of $L$ such that
\begin{enumerate}
    \item \emph{(Conditional distribution on selected intervals)} Conditioning on a selected interval, the discrepancy of the expected value between $U$ and $V$ is controlled by $2\epsilon/\delta$.
 That is, for any $\iteratorInterval \in [\numInterval]$, 
    \begin{align}\label{eq:distribution-shift}
        \Big|\E\Big[U\mid V \in [a_{i_\iteratorInterval},b_{i_\iteratorInterval}]\Big] - \E\Big[V \mid V \in [a_{i_\iteratorInterval},b_{i_\iteratorInterval}]\Big]\Big| \leq 2\epsilon/\delta.
    \end{align}
    \item \emph{(Coverage guarantee)}
    The intervals not selected carry at most a $\delta$-fraction of the probability mass. That is,
\begin{align}\label{eq:distribution-shift-coverage}
        \PP\Big(V \in \cup_{\iteratorInterval \in [L] \setminus \{i_1,\dots,i_\numInterval\}} [a_\iteratorInterval,b_\iteratorInterval]\Big) \leq \delta.
    \end{align}
    Consequently, the union of the selected intervals has probability mass $1-2\delta$.
\end{enumerate}
\end{lemma}

\begin{proofsketch}
By applying the triangle inequality to each interval $w$ violating Eq.~\eqref{eq:distribution-shift}, the conditional mean absolute deviation satisfies $\mathbb{E}[\vert{}U - V\vert{} \mid V \in [a_w, b_w]] > 2\epsilon/\delta$. Accounting for endpoint overlaps ($\sum_w \mathbbm{1}(V \in [a_w, b_w]) \leq 2$), we have $2\epsilon \ge \sum_w \mathbb{E}[|U-V| \mathbbm{1}(V \in [a_w, b_w])] > \mathbb{P}(V \in \bigcup_w [a_w, b_w]) \cdot (2\epsilon/\delta)$, yielding total violating mass bounded by $\delta$. The full proof is provided in Supplementary Material~\ref{proof_lem:conditional-dist-shift}.
\end{proofsketch}

\begin{remark}[Extension to regression] 
\label{rmk:extension_regression}
If $p$ is a regression predictor intended to approximate the conditional mean
$
\mathbb{E}[Y\mid X=x],
$ Definition~\ref{def:picpi} becomes 
\[
\mathbb{E}\!\left[Y\mid p(X)\in I_j\right]\in I_j.
\]
For bounded outcomes where $Y\in [a,b]$ almost surely, the constructions in Section~\ref{sec:def-alg} still apply by replacing the empirical positive-outcome frequency with the empirical average of $Y$.
The finite-sample argument in \Cref{thm:consistency-prediction} holds via an affine rescaling of $Y$ and $p(X)$ from $[0,1]$ to $[a,b]$. For unbounded outcomes, the definition remains valid whenever the relevant
conditional expectations exist, but the concentration
guarantees developed here do not apply without additional tail or moment assumptions. 
\end{remark}

\section{Inference via \titleShorts}\label{sec:tasks}

We now show how to use \titleShorts as building blocks for prediction sets and inference in two tasks: interval prediction and label set prediction.

\subsection{Interval prediction for probabilities}\label{sec:task_1}

\titleShorts can be used to construct intervals for probabilistic prediction with coverage guarantees.
For test input $x\in\X$, we output  an $\epsilon$-expanded \titleShort that contains $p(x)$. The procedure is formalized as follows: 
define functions $l,u: [0,1] \to [0,1]$ such that the interval $[l(z),u(z)]$ represents a \titleShort that contains $z$. 
Define the expansion operator $E^\epsilon$ such that $E^\epsilon(A) = \{b \in [0,1]: d(b,A) \leq \epsilon\}$ for any set $A \subset [0,1]$ and Euclidean distance $d$. 
The prediction interval is:
\begin{align}\label{eq:task-predicting_true_prob}
    C_{\tp}(x;\epsilon) = E^\epsilon\Big(\Big[l(p(x)),u(p(x))\Big]\Big).
\end{align}

The following theorem establishes how well $C_{\tp}(x;\epsilon)$ captures the true conditional probability. Specifically, the coverage error is upper-bounded by the fluctuation of true probabilities within each interval and the uncovered prediction mass. By Lemma~\ref{lem:dist-regularity} the second term is at most $\delta$.



\begin{theorem}[Marginal coverage of the true probability]\label{thm:tp}
Let $p^*(x)=\PP(Y=1\mid X=x)$ denote the true probability that $Y=1$
given $x$. Assume we have disjoint \titleShorts
$\{\interval_j\}_{j=1}^{\numInterval}$, and write
$m_j=\PP(p(X)\in\interval_j)$. For any $\epsilon>0$, it holds that
\begin{align*}
  \PP\Big(p^*(X)\notin C_{\tp}(X;\epsilon)\Big) \leq\epsilon^{-2}
 \underbrace{\sum_{j:m_j>0}m_j
 \Var\Big(p^*(X)\mid p(X)\in\interval_j\Big)}_{\text{``in-bin'' variation}}
 +\underbrace{\Big(1-\sum_{j=1}^{\numInterval}m_j\Big)}_{
 \text{predictions not covered by }\intervalset_p}.
\end{align*}
\end{theorem}

When the predictor and interval collection are constructed from data that is independent of the test observation, the theorem applies conditionally on those data whenever the resulting intervals satisfy its hypotheses. Its proof can be found in Supplementary Material~\ref{proof_thm:tp}.

The bound in Theorem~\ref{thm:tp} can be sharpened when additional structure is available. The following corollary shows that if the predictor $p$ approximates the true probability $p^*$ uniformly or satisfies a Lipschitz-type condition, the in-bin variation term admits explicit upper bounds in terms of the interval lengths.



\begin{corollary}
[Coverage under smoothness conditions] \label{cor:tp_smoothness}
Under the assumptions of Theorem~\ref{thm:tp}, suppose that
$\|p^*-p\|_\infty\leq\eta$ for some $\eta\geq0$, then
\begin{align*}
 \PP\Big(p^*(X)\notin C_{\tp}(X;\epsilon)\Big) 
 &~\leq\frac{1}{4\epsilon^2}\sum_{j:m_j>0}
 m_j(|\interval_j|+2\eta)^2
 +\Big(1-\sum_{j=1}^{\numInterval}m_j\Big).
\end{align*}
Alternatively, suppose that
$|p^*(x)-p^*(x')|\leq L\cdot|p(x)-p(x')|$ for all $x,x'\in\X$, then
\begin{align*}
 \PP\Big(p^*(X)\notin C_{\tp}(X;\epsilon)\Big)
 &~\leq\frac{L^2}{4\epsilon^2}\sum_{j:m_j>0}
 m_j|\interval_j|^2
 +\Big(1-\sum_{j=1}^{\numInterval}m_j\Big).
\end{align*}
\end{corollary}

We provide the proof in Supplementary Material~\ref{proof_cor:tp_smoothness}.

\subsection{Label set prediction for multi-class classification}
\label{sec:task_2}
We now use \titleShorts for label set prediction in classification with $\numClass$ classes. In particular, we establish three guarantees: (i) an exact class-conditional coverage certificate (Theorem~\ref{thm:mc_error}); (ii) a population size-optimality, showing that our interval selection rule minimizes expected label-set size 
(Theorem~\ref{thm:mc_optimality}); and (iii) a finite-sample same-sample guarantee, showing that an empirically feasible procedure retains valid class-conditional coverage when the same calibration data are used to construct the intervals and estimate their masses (Theorem~\ref{thm:mc_same_sample}).


Let
$
p:\X\rightarrow[0,1]^\numClass
$
be a prediction model, where $p_\iteratorClass(x)$ predicts
$\PP(Y(X)=\iteratorClass\mid X=x)$. For each class
$\iteratorClass\in[\numClass]$, we obtain disjoint \titleShorts
$
\intervalset_{p_\iteratorClass}
=
\bigl\{[a_{\iteratorClass,j},b_{\iteratorClass,j}]\bigr\}_{j=1}^{J_\iteratorClass}
$
for the one-vs-rest problem $1\{Y(X)=\iteratorClass\}$.
We first construct an acceptance region $A_\iteratorClass(S_\iteratorClass)$ from an interval subset selection rule $S_\iteratorClass\subseteq[J_\iteratorClass]$, then include a class label if its prediction is in acceptance region $A_\iteratorClass(S_\iteratorClass)$.
That is, the label set $C_{\mc}(x;S)$ is
\begin{align}
C_{\mc}(x;S)
:=
\Bigl\{
\iteratorClass\in[\numClass]:
p_\iteratorClass(x)\in A_\iteratorClass(S_\iteratorClass)
\Bigr\},  \qquad A_\iteratorClass(S_\iteratorClass)
:=
\bigcup_{j\in S_\iteratorClass}
[a_{\iteratorClass,j},b_{\iteratorClass,j}].
\label{eq:mc_acceptance}
\end{align}

Here $S=(S_1,\ldots,S_\numClass)$.
For notational convenience, let
\begin{align*}
m_{\iteratorClass,j}
:=
\PP\bigl(
p_\iteratorClass(X)\in[a_{\iteratorClass,j},b_{\iteratorClass,j}]
\bigr),
\qquad
\pi_\iteratorClass
:=
\PP(Y(X)=\iteratorClass), \qquad t_\iteratorClass
:=
1-\sum_{j=1}^{J_\iteratorClass}m_{\iteratorClass,j}.
\end{align*}
Throughout, assume $\pi_\iteratorClass>0$ for all
$\iteratorClass\in[\numClass]$.

\ifjrssb
\textbf{Class-conditional coverage. }
\else
\paragraph{Class-conditional coverage. }
\fi
The following theorem gives an exact coverage certificate under a subset selection rule $S$.
It combines an upper bound on excluded true-class mass, and a lower bound on the captured true-class mass.

\begin{theorem}[Class-conditional coverage of the true label]\label{thm:mc_error}
Fix a class $\iteratorClass\in[\numClass]$ and a subset $S_\iteratorClass\subseteq [J_\iteratorClass]$.
It holds that 
\begin{align}
\PP\Bigl(Y(X)\notin C_{\mc}(X;S)\mid Y(X)=\iteratorClass\Bigr)
\le
\Gamma_{\iteratorClass}(S_\iteratorClass),
\end{align}
where
\[
\Gamma_{\iteratorClass}(S_\iteratorClass)
:=
\min\Bigl\{
\Gamma_{\iteratorClass}^{\mathrm{up}}(S_\iteratorClass),
\Gamma_{\iteratorClass}^{\mathrm{low}}(S_\iteratorClass)
\Bigr\},
\]
and 
\[
\Gamma_{\iteratorClass}^{\mathrm{up}}(S_\iteratorClass)
:=
\frac{t_\iteratorClass+\sum_{j\notin S_\iteratorClass} m_{\iteratorClass,j} b_{\iteratorClass,j}}{\pi_\iteratorClass},
\qquad
\Gamma_{\iteratorClass}^{\mathrm{low}}(S_\iteratorClass)
:=
1-\frac{\sum_{j\in S_\iteratorClass} m_{\iteratorClass,j} a_{\iteratorClass,j}}{\pi_\iteratorClass}.
\]
Moreover, the bound is exact: given the interval masses and class probabilities, \(\Gamma_{\iteratorClass}(S_\iteratorClass)\) is the largest miscoverage among all feasible choices of $\mu_j$ such that
\[
\mu_j
=
\PP\!\left(
Y(X)=\iteratorClass
\mid
p_\iteratorClass(X)\in[a_{\iteratorClass,j},b_{\iteratorClass,j}]
\right)
\in[a_{\iteratorClass,j},b_{\iteratorClass,j}].
\]
\end{theorem}
In other words, this is the  tightest bound given only the information encoded by the \titleShort condition. For the proof see
Supplementary Material~\ref{proof_thm:mc_error}.

\ifjrssb
\textbf{Population-optimal label sets. }
\else
\paragraph{Population-optimal label sets.}
\fi
Ideally, we want our predicted label set to be as small as possible while still guaranteeing that the true label is caught within error budgets.
We now present an interval selection rule to form an optimal acceptance region:
let $\alpha = (\alpha_1, \ldots, \alpha_\numClass) \in [0,1]^\numClass$ denote a vector of target class-conditional miscoverage levels.
For each class $\iteratorClass \in [\numClass]$, we select an optimal interval subset  by minimizing their mass subject to the error budget. Formally, it is defined as:
\begin{align}
S_\iteratorClass^\star(\alpha_\iteratorClass)
\in
\arg\min_{S\subseteq[J_\iteratorClass]}
\sum_{j\in S} m_{\iteratorClass,j}
\quad \text{subject to} \quad
\Gamma_{\iteratorClass}(S)\le \alpha_\iteratorClass.
\label{eq:mc_population_opt}
\end{align}
The resulting optimal label set is constructed from $S_\iteratorClass^\star(\alpha_\iteratorClass)$ using Eq.~\eqref{eq:mc_acceptance}:
\[
C_{\mc}^\star(x;\alpha)
:=
C_{\mc}\bigl(x;S^\star(\alpha)\bigr),
\qquad
S^\star(\alpha)
:=
\bigl(
S_1^\star(\alpha_1),\ldots,S_\numClass^\star(\alpha_\numClass)
\bigr).
\]
The optimization in Eq.~\eqref{eq:mc_population_opt} decomposes into two
simpler subset-selection problems corresponding to
$\Gamma_{\iteratorClass}^{\mathrm{up}}$ and
$\Gamma_{\iteratorClass}^{\mathrm{low}}$, of which we supply the explicit formulation and algorithm in the Supplementary Material~\ref{app:mc_details_optimization}.

\begin{theorem}[Size-optimality at fixed coverage]
\label{thm:mc_optimality}
For every class $\iteratorClass\in[\numClass]$,
\[
\PP\Bigl(
Y(X)\notin C_{\mc}^\star(X;\alpha)
\,\Big|\,
Y(X)=\iteratorClass
\Bigr)
\le
\alpha_\iteratorClass.
\]
Moreover,
\[
\E\Bigl[
\bigl|C_{\mc}^\star(X;\alpha)\bigr|
\Bigr]
=
\min_S \Bigl\{
\E\bigl[|C_{\mc}(X;S)|\bigr]:
\Gamma_{\iteratorClass}(S_\iteratorClass)
\le\alpha_\iteratorClass
\ \text{for all }\iteratorClass
\Bigr\}.
\]
Thus, given a fixed collection of classwise \titleShorts, $C_{\mc}^\star$ has the smallest expected label-set size among all
deterministic unions of the supplied intervals satisfying the prescribed
class-conditional guarantees.
\end{theorem}

The proof is given in
Supplementary Material~\ref{proof_thm:mc_optimality}.

\ifjrssb
\textbf{Certificate for feasible implementation. }
\else
\paragraph{Certificate for feasible implementation.}
\fi
The population selection rule depends on the unknown interval masses
$m_{\iteratorClass,j}$ and class probabilities $\pi_\iteratorClass$.
The key point is that the same calibration sample may be used both to construct \titleShorts and to estimate these quantities.
With these quantities, we derive $\widehat\Gamma_{\iteratorClass}(S)$, a conservative miscoverage estimate, 
and then follow the same procedure in Eq.~\eqref{eq:mc_population_opt} to construct $\widehat S_\iteratorClass^\star(\alpha_\iteratorClass)$ and $\widehat C_{\mc}^\star(x;\alpha)$.
A formalization of this process and the proof of the following certificate can be found in Supplementary Material~\ref{app:mc_details_same_sample} and~\ref{proof_thm:mc_same_sample}, respectively.

\begin{theorem}[Fully feasible same-sample certificate]\label{thm:mc_same_sample}
Suppose we obtain a set of retained intervals from Eq.~\eqref{eq:mc_population_opt}
with classwise failure budgets summing to $\delta_{\mathrm{PI}}$, so that with probability at least $1-\delta_{\mathrm{PI}}$ all retained intervals in
\[
\widehat{\intervalset}_{p_\iteratorClass}
=
\bigl\{
\widehat I_{\iteratorClass,j}
=
[\widehat a_{\iteratorClass,j},\widehat b_{\iteratorClass,j}]
\bigr\}_{j=1}^{\widehat J_\iteratorClass},
\qquad
\iteratorClass\in[\numClass],
\]
are disjoint \titleShorts simultaneously.
For any $S=(S_1,\ldots,S_\numClass)$ with $S_\iteratorClass\subseteq [\widehat J_\iteratorClass]$, define
\[
\widehat C_{\mc}(x;S)
=
\Bigl\{
\iteratorClass\in[\numClass]:
p_\iteratorClass(x)\in
\bigcup_{j\in S_\iteratorClass}\widehat I_{\iteratorClass,j}
\Bigr\}.
\]
Then, with probability at least
$
1-\delta_{\mathrm{PI}}-\delta_m-\delta_\pi
$
over the draw of calibration dataset $D_{\mathrm{cal}}$, it holds simultaneously for all classes $\iteratorClass\in[\numClass]$ and all subsets $S_\iteratorClass\subseteq[\widehat J_\iteratorClass]$ that
\[
\PP\Bigl(
Y(X)\notin \widehat C_{\mc}(X;S)
\ \Big|\ 
Y(X)=\iteratorClass,\ D_{\mathrm{cal}}
\Bigr)
\le
\widehat\Gamma_{\iteratorClass}(S_\iteratorClass).
\]
In particular, whenever $\widehat S_\iteratorClass^\star(\alpha_\iteratorClass)$ is feasible for every class,
\[
\PP\Bigl(
Y(X)\notin \widehat C_{\mc}^\star(X;\alpha)
\ \Big|\ 
Y(X)=\iteratorClass,\ D_{\mathrm{cal}}
\Bigr)
\le
\alpha_\iteratorClass
\qquad
\text{for all }\iteratorClass\in[\numClass].
\]
\end{theorem}

\section{Experiments}\label{sec:experiment}


We compare \titleShort-based methods with other interval-prediction methods for predicting probabilities (Task 1) and classification (Task 2). In both tasks, point predictors are fitted on a training split, intervals (or label sets) are constructed using a separate calibration split, and the resulting intervals (or label sets) are evaluated on a test split. 


For Task 1, as described in Section~\ref{sec:task_1}, the goal is to produce an interval to predict $\PP (Y=1\mid X=x)$. 
The data-generating processes (DGPs) use 1-D and 20-D covariates with binary targets.
For Task 2, as described in Section~\ref{sec:task_2}, the goal is to produce a label set. 
For DGPs, target labels are drawn as $Y \mid X=x \sim \mathrm{Categorical}(q(x))$. The Linear-Gaussian DGP uses $X \sim \mathcal{N}(0, I_{10})$ with sparse linear softmax probabilities $q(x) = \mathrm{softmax}(Wx + b)$. The Nonlinear-Mixture DGP generates non-Gaussian, multimodal covariates via a four-component latent-factor mixture. Its class probabilities scale the linear logits by a bounded covariate-dependent function $s(x)$ such that $q(x) = \mathrm{softmax}\left((Wx + b)/{s(x)}\right)$. Implementation details are available in Supplementary Material~\ref{app:task1_setup} and~\ref{app:task2_setup}, respectively.

\subsection{Baselines}

We compare the following baselines: (i) \titleShorts. 
(ii) \textbf{Simultaneous confidence intervals} obtained by applying the Delta method to a Wald confidence region. 
(iii) \textbf{Split conformal prediction} forms symmetric intervals around the fitted probability using quantiles of absolute residuals, providing marginal coverage guarantees. 
(iv) \textbf{Calibration-based intervals} discretize $[0,1]$ into bins and retain the ones with calibration error comparable to \titleShort.
(v) \textbf{Fixed-width binning} discretizes $[0,1]$ into fixed-width bins and reports the bin containing the predicted probability. (vi) \textbf{Conformal classification}, used only for Task~2, outputs classes whose predictions exceed a calibration-based threshold. This is not an interval-based method, but has marginal coverage of the true label.
We extend their formulation details and metrics in Supplementary Material~\ref{app:baselines} and~\ref{app:metrics}, respectively.

\subsection{PICPIs with empirical self-consistency}
\label{sec:alg_emp}
\begin{algorithm}[tb] 
   \caption{ \titleShort Construction (empirical mode with empirical self-consistency)}
   \label{alg:calibration_empirical} 

\begin{algorithmic}[1]
\Procedure{\textsc{Calibration}}{data points $\{(x_i,y_i)\}_{i=1}^\numData,$ predictor $ p$} \Comment{Obtain a set of \titleShorts with respect to the empirical distribution}
\State Compute predictions $p_i\leftarrow p(x_i)$ and sort $p_i$ in ascending order. Let $p_i'$ and $y_i'$ to denote the sorted prediction and outcome respectively.
\State Let $t_0 \leftarrow 0$.

\For {$j = 1,2,\dots,$}
\State Find $\min t_j \in (t_{j-1},1] $ subject to \Comment{find the next endpoint on the grid}
\begin{align*}
    \frac{\sum_{i=1}^n y_i' \mathbbm{1}(p_i' \in (t_{j-1}, t_j])}{\sum_{i=1}^n \mathbbm{1}(p_i' \in (t_{j-1}, t_j])} \in ~ (t_{j-1}, t_j] \quad \text{ and }\quad 
    \frac{\sum_{i=1}^n y_i' \mathbbm{1}(p_i' \in (t_{j}, 1])}{\sum_{i=1}^n \mathbbm{1}(p_i' \in (t_{j}, 1])} \in ~ (t_{j}, 1].
\end{align*}
\State Let $\interval_j \leftarrow (t_{j-1}, t_j]$.
\If {$t_j = 1$}
\State Let $\numInterval \leftarrow j$. {\bf Break}.
\EndIf
\EndFor
\State {\bf Return} $\{\interval_j\}_{j=1}^\numInterval$
\EndProcedure
\end{algorithmic}
\end{algorithm}

We now present Algorithm~\ref{alg:calibration_empirical}, the empirical variant of algorithm used in experiments in this section. It returns \titleShorts with respect to the empirical distribution of the calibration sample. We refer to this property as \emph{empirical self-consistency} to distinguish it from the population validity guaranteed by Theorem~\ref{thm:algorithm_population}. 
The key difference is that Algorithm~\ref{alg:calibration_empirical} omits the concentration margin (Line~\ref{line:empirical-interval}, Algorithm~\ref{alg:calibration_population}) when evaluating intervals.  
While this omission yields narrower, more practical intervals, the theoretical population validity guaranteed by Theorem~\ref{thm:algorithm_population} no longer applies.
Figure~\ref{fig:emp_mode_diagnostics} illustrates the resulting tradeoff. 
Most proposed intervals also satisfy the \titleShort condition on held-out data, whereas the shortest intervals (of length at most $0.03$) are more prone to false acceptance.

The algorithm proceeds as follows: 
starting at 0, it chooses the smallest next endpoint for which the \emph{empirical} label mean of the new bin lies in that bin, and so does the remaining part in $[0,1]$. It then repeats this step on the remainder of $[0,1]$.

\begin{figure}[tb]
    \centering
    \vspace{3mm}
    \begin{subfigure}[b]{0.36\textwidth}
        \centering
        \includegraphics[width=\textwidth]{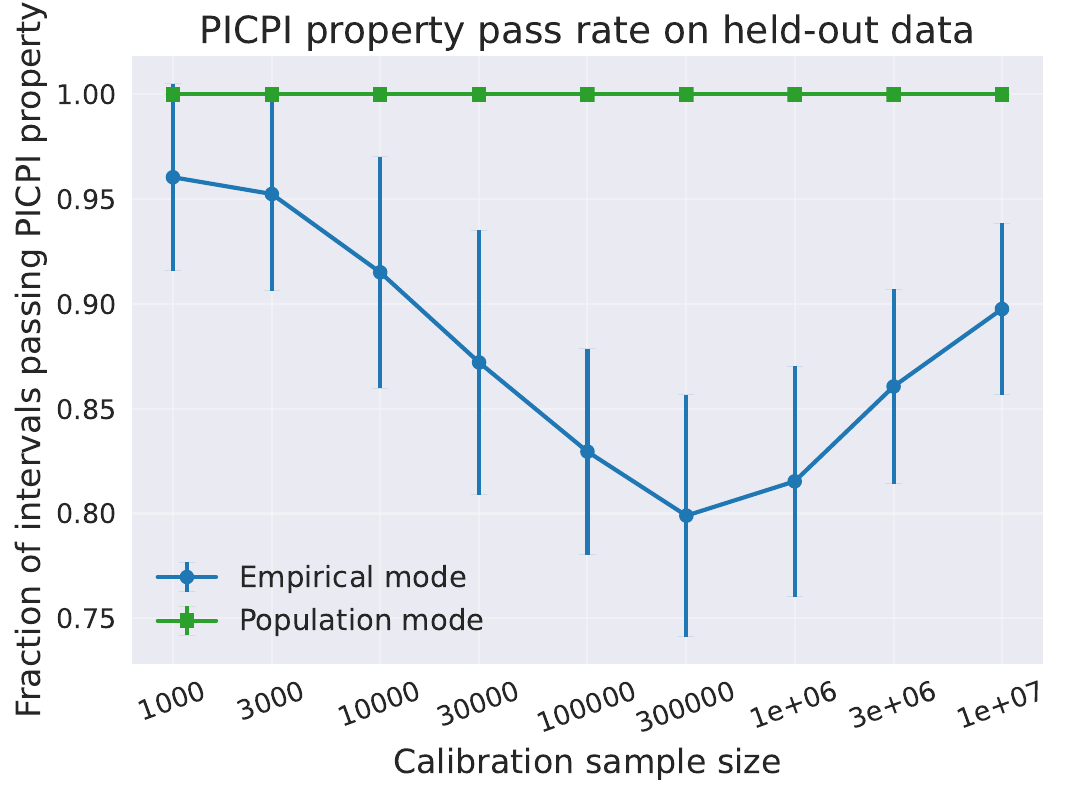}

    \end{subfigure}
    \hfill 
    \begin{subfigure}[b]{0.6\textwidth}
        \centering
        \includegraphics[width=\textwidth]{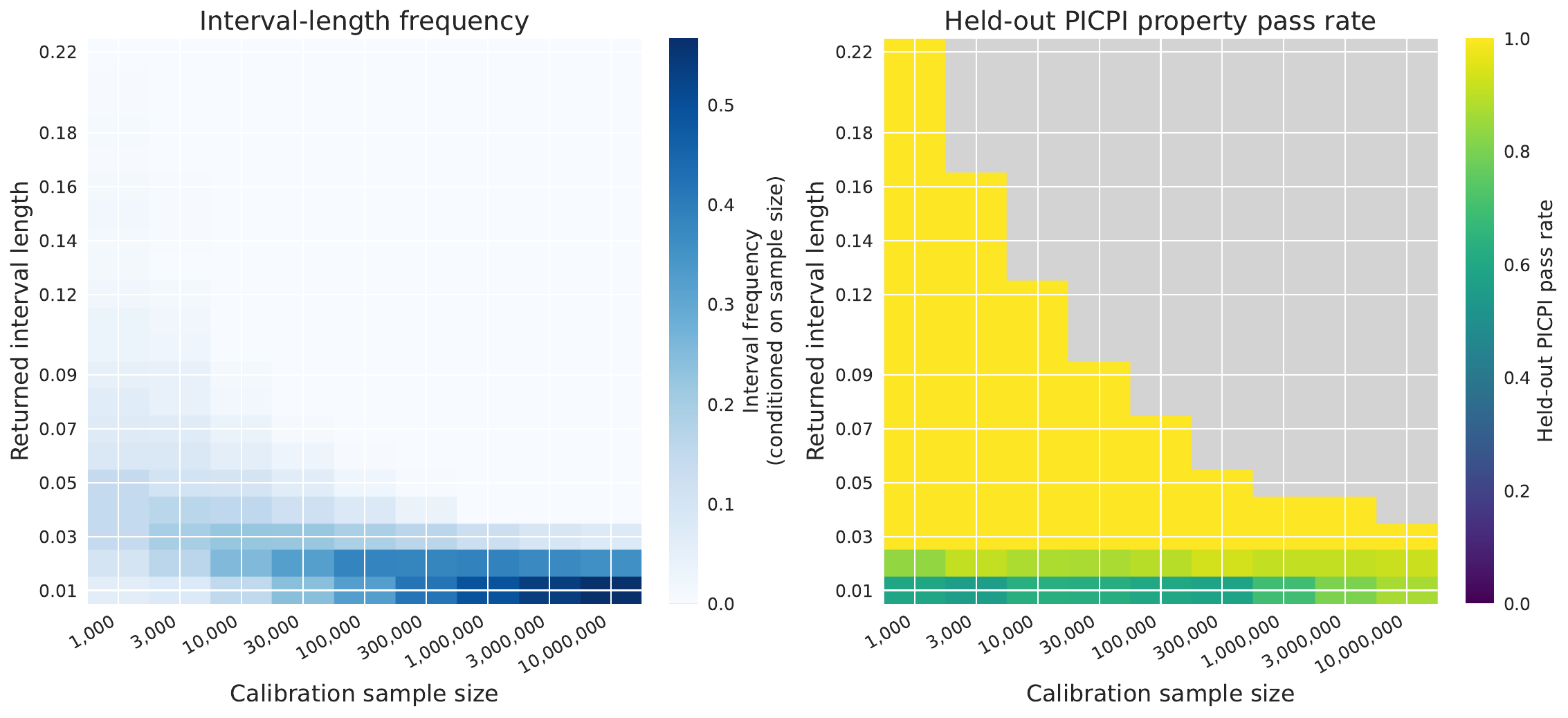}

    \end{subfigure}
    
    \caption{Demonstration of empirical version (Algorithm~\ref{alg:calibration_empirical}). \textbf{Left:} Proportion of proposed intervals satisfying the \titleShort property on a held-out dataset. \textbf{Middle:} With more calibration data, the empirical version admits narrower intervals. \textbf{Right:} While most intervals satisfy the \titleShort property, the narrowest intervals (length $\leq 0.03$) at the bottom exhibit lower pass rates. Their pass rate follows a U-shaped pattern as the calibration data size increases: initially, additional data allow the procedure to admit increasingly narrow and more challenging intervals, leading to a lower pass rate; with sufficiently large samples, however, the increased information improves calibration even for these narrow intervals, and the pass rate recovers.
     For a formal presentation of the setup studied here, see Supplementary Material~\ref{app:task1_setup_emp_mode_diagnostics}. }
    \label{fig:emp_mode_diagnostics}
\end{figure}

\begin{figure}[!ht]
\centering

\includegraphics[width=\linewidth]{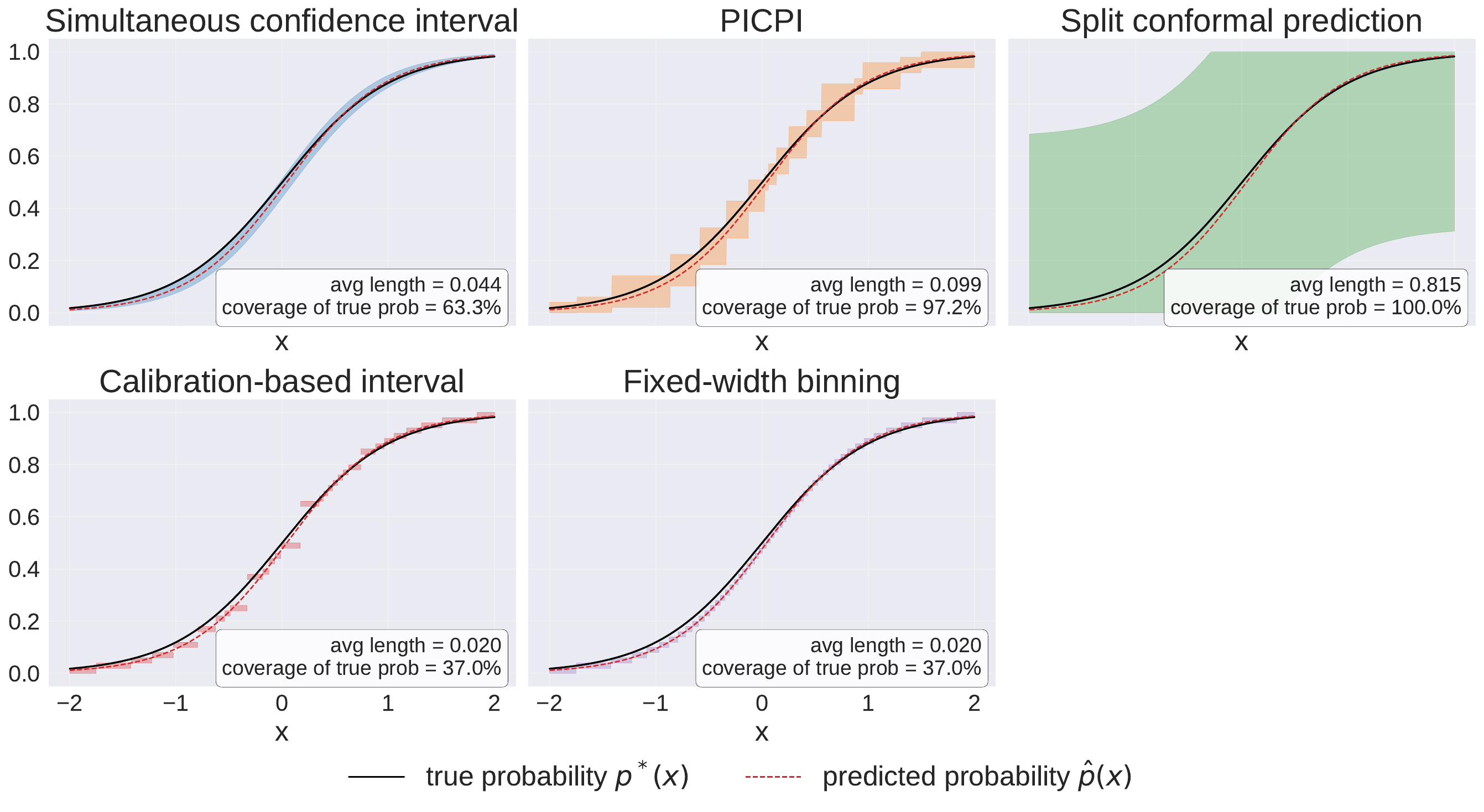}
\caption{Visualization under mildly miscalibrated model. 
\titleShort adapts its interval endpoints to local prediction bins to account for deviations between the fitted and target probabilities, while maintaining substantially shorter intervals than split conformal prediction. The simultaneous confidence interval is less responsive to bin-specific deviations from calibration, while the binning-based baselines have poor coverage of the target probability. 
}
\label{fig:task1_interval_visualization} 
\end{figure}

\subsection{Task 1 Results}

Figures~\ref{fig:task1_interval_visualization}, together with Table~\ref{tab:task1_multivariate_mc_summary}, show that \titleShort achieves a favorable balance among length, calibration, and coverage measures in the current setting. 
Quantitatively,  \titleShort attains the lowest calibration errors while maintaining substantially shorter intervals than the simultaneous confidence interval and split conformal prediction, and it achieves much higher true-probability coverage than the calibration-based and fixed-width binning baselines (Table~\ref{tab:task1_multivariate_mc_summary}). 
In particular,
its calibration errors indicate that the reported intervals are well centered around the true probabilities and the its binning follows the calibration structure of predictor. 


\ifjrssb
\else
\begin{table}[ht]
\centering

\begin{tabular}{lccccc}

\toprule
Method & Length & ECE$_\text{mid}$ &  ECE$_\text{pred}$  & Coverage$_{[0,1]}$ & Coverage$_{p^*}$ \\
\midrule
Simultaneous CI & 0.41 & 6.39 & 3.25 & -- & \textbf{100.00\%} \\
PICPI & 0.11 & \textbf{1.00} & \textbf{1.00} & \textbf{100.00\%} & 79.47\% \\
Split conformal prediction & 0.87 & 27.02 & 2.43 & -- & \textbf{100.00\%} \\
Calibration-based interval & \textbf{0.02} & 1.16 & 1.01 & 60.78\% & 28.07\% \\
Fixed-width binning & \textbf{0.02} & 1.05 & 1.11 & \textbf{100.00\%} & 29.28\% \\
\bottomrule

\end{tabular}
\caption{Comparison of interval lengths, calibration errors, and coverages. \titleShort achieves substantially shorter intervals than simultaneous and conformal methods while maintaining full coverage of the prediction space $[0,1]$ and improved coverage of the target probability $p^*$.
$\mathrm{ECE_{mid}}$ measures if interval midpoint is representative of the true probabilities of the predictions assigned to it.
$\mathrm{ECE_{pred}}$ measures if average prediction agree with the true probability among the points receiving this interval. It evaluates how well the interval grouping preserves the calibration of $p(X)$. Formal definitions of the metrics are provided in Supplementary Material~\ref{app:metrics}. Covariate dimension $d=20$. }
\label{tab:task1_multivariate_mc_summary}
\end{table}
\fi

\subsection{Task 2 Results}

\begin{figure}[ht]
\centering
\includegraphics[width=\linewidth]{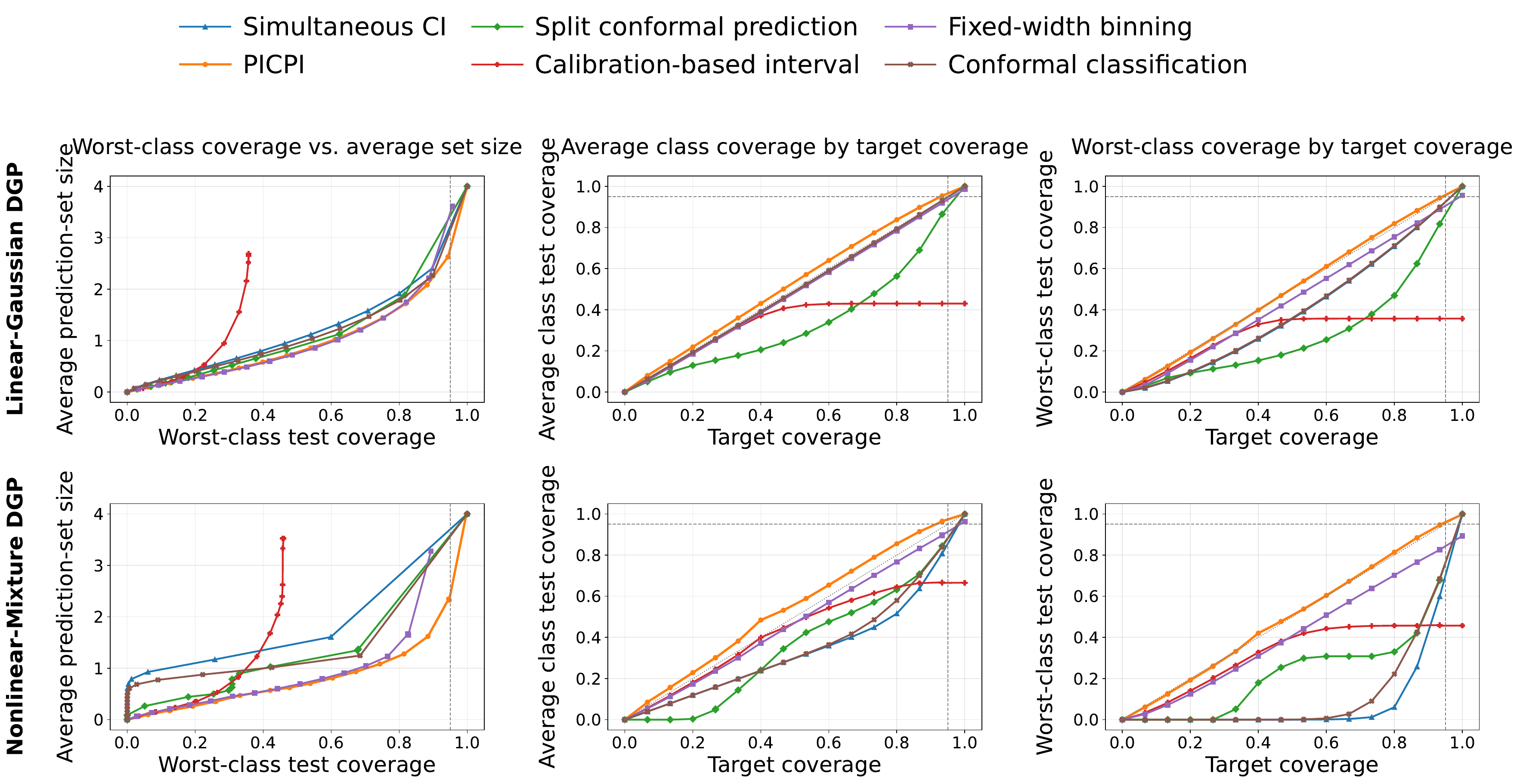}
\caption{Task~2 results, averaged over $1000$ replications. First row: linear-Gaussian DGP. Second row: nonlinear-mixture softmax DGP. Each replication uses disjoint splits of size $n_{\mathrm{train}}=n_{\mathrm{cal}}=n_{\mathrm{eval}}=2000$, with $d=10$ covariates and $\numClass=4$ classes. \titleShort maintains near-target worst-class coverage in both settings. This advantage is more pronounced under the more complex nonlinear-mixture DGP, where competing baselines show substantial degradation in worst-class coverage. Implementation details are given in Supplementary Material~\ref{app:task2_setup}.}
\label{fig:task2}
\end{figure}

Across both the linear-Gaussian and nonlinear-mixture softmax DGPs, \titleShort delivers the most stable worst-class coverage relative to prediction-set size, as illustrated in the first column of Figure~\ref{fig:task2}. This advantage becomes increasingly pronounced as the complexity of the DGP grows. The second and third columns of Figure~\ref{fig:task2} demonstrate that under the nonlinear-mixture DGP, the performance of competing methods deteriorates substantially. In contrast, \titleShort continues to track the target coverage closely while maintaining the smallest prediction sets. 


\section{Related Work}\label{sec:related_work}

\ifjrssb
\textbf{Conformal prediction.}
\else 
\paragraph{Conformal prediction.}
\fi
Conformal prediction constructs set-valued predictors 
with marginal validity (Eq.~\eqref{eq:intro_conformal_marginal})
under exchangeability \citep{vovk2005algorithmic,shafer2008tutorial}. The framework can be layered on top of a broad range of base learners \citep{vovk2015cross,romano2019conformalized,kim2020predictive}.
A stronger target is conditional coverage (Eq.~\eqref{eq:conditional_conformal}),
but such distribution-free guarantees are impossible in the worst case except with vacuous sets \citep{lei2014distribution,barber2021limits}. Recent work therefore studies weaker conditional notions obtained by coarsening the conditioning event: Mondrian or group-conditional validity conditions on a prespecified partition or label \citep{vovk2003mondrian,vovk2012conditional,martinezgil2024groups}; multivalidity asks for simultaneous approximate validity over subpopulations \citep{jung2023multivalid,gibbs2025conformal}; and more recently sample- and training-conditional targets \citep{duchi2025few,liang2025algorithmic}.
The \titleShort target belongs to this ``nearly conditional'' regime, but conditioning on prediction instead of the covariates. 
Theorem~\ref{thm:tp} converts this target into a conformal-style guarantee, and Theorem~\ref{thm:consistency-prediction} supplies an efficiency statement on width shrinkage and coverage. This complements a separate line of work that studies the length, volume, or oracle optimality of conformal regions 
\citep{lei2013distribution, chernozhukov2021distributional, izbicki2020flexible, yang2025selection,lebars2025volume}.  


\ifjrssb
\textbf{Calibration.}
\else
\paragraph{Calibration.}
\fi
The pointwise calibration condition of \citet{dawid1982wellcalibrated} and \citet{foster1998asymptotic} is typically assessed via coarser summaries, such as reliability diagrams or aggregate metrics like ECE
\citep{gneiting2007strictly,vaicenavicius2019evaluating,carrell2022calibration,blasiok2024smooth}. Post-hoc calibration literature then modifies the original prediction
to improve this global agreement 
\citep{platt1999probabilistic,naeini2015obtaining,kull2017beta}.
In comparison, \titleShorts do not alter the model, but instead seek locally calibrated stratum. In fact, the model $p$ may also be the output of a recalibration procedure.
On the other hand, calibrated probabilities can be coupled with conformal prediction for uncertainty statements \citep{gupta2020distribution, podkopaev2021distribution, van2024self}. 
Multicalibration asks for approximate calibration over subgroups 
\citep{kleinberg2016inherent,pleiss2017fairness, hebert2018multicalibration}. \titleShort can be read as a one-dimensional, prediction-indexed analogue: the conditioning sets are subsets on the prediction axis rather than externally specified demographic groups, and the guarantee is interval-valued rather than an additive calibration error bound. 


\section{Conclusions}
\label{sec:conclusions}

In summary, we have introduced \titleShorts as a framework for local, data-driven uncertainty quantification that bridges the gap between calibration and conformal prediction. By conditioning on predictions, \titleShorts provide intervals that are both interpretable and adaptive to the precision of the underlying predictor. Unlike traditional calibration methods, which rely on fixed bins or post-hoc adjustments, our approach identifies regions of the prediction space where the model is self-certified to be reliable, yielding uncertainty statements that are directly actionable without modifying the base predictor. We show that, under suitable regularity conditions, the width of these intervals shrinks with increasing sample size, enabling increasingly precise risk characterization. 

Practically, \titleShorts offer a tool for probabilistic prediction and multi-class classification with respective coverage guarantees. Overall, our results highlight that leveraging the model’s own risk stratification can produce rigorous, locally adaptive uncertainty quantification, providing an alternative to conventional recalibration or fully conditional prediction frameworks. This opens up directions on extending \titleShorts to structured settings inherently require prediction-conditioned uncertainty quantification, such as representation-based inference (where the predictions acts as a low-dimensional summary), covariate shift, and distribution shift (where prediction-induced strata captures model reliability).
On the theoretical side, a future direction is to develop adaptive rates on the width bound or to relax the regularity assumption. Another direction is to improve the concentration margin in Algorithm~\ref{alg:calibration_population} using alternative inequalities or empirical process.



\ifjrssb
\newpage

\fi
\clearpage
\bibliographystyle{\xuelinbibstyle}
\bibliography{ref}

\clearpage
\appendix

\begin{center}
    {\large\bf Supplementary Materials }
\end{center}
\section{Proofs}

\subsection{Proof of Theorem~\ref{thm:algorithm_population}}
\label{proof_thm:algorithm_population}

\begin{proof}

For each discretized interval $[a, b]$ with empirical count $N_{[a,b]} = \sum_i \mathbbm{1}\{p(x_i) \in [a,b]\} > 0$, we apply Hoeffding's inequality and a union bound to derive that with probability at least $1-\delta$:
\begin{align*}
    |\hat{\mu}_{[a,b]} - \mu_{[a,b]}| \leq 2\sqrt{\frac{\log(|\mathcal{I}|/\delta)}{N_{[a,b]}}},
\end{align*}
where $\mu_{[a,b]} = \PP(Y=1 \mid p(X) \in [a, b])$,
\begin{align*}
    \hat{\mu}_{[a,b]} = \frac{\sum_{i=1}^n y_i \cdot \mathbbm{1}(p(x_i) \in [a,b])}{\sum_{i=1}^n \mathbbm{1}(p(x_i) \in [a,b])}
\end{align*}
and $|\mathcal{I}| = K^2$ is the number of candidate intervals.

On the event that all these bounds hold simultaneously, Line~\ref{line:empirical-interval} implies that
\begin{align*}
    \mu_{[a,b]} \in [a,b].
\end{align*}
This confirms that with probability at least $1-\failureProb$, $\intervalset_p$ is a set of \titleShorts.

We note that a relaxed alternative for Line~\ref{line:empirical-interval} is to add indicator functions at the boundary endpoints 0 and 1. That is, 
{
\begin{align}
    & \frac{\sum_{i=1}^n y_i \cdot \mathbbm{1}(p(x_i) \in [a,b])}{\sum_{i=1}^n \mathbbm{1}(p(x_i) \in [a,b])} \nonumber\\
    &\quad \in \left[a + 2\sqrt{\frac{\log (\discreteNum^2/\delta)}{\sum_{i=1}^n \mathbbm{1}(p(x_i) \in [a,b])}} \mathbbm{1}\{a=0\} ,b - 2\sqrt{\frac{\log (\discreteNum^2/\delta)}{\sum_{i=1}^n \mathbbm{1}(p(x_i) \in [a,b])}}\mathbbm{1}\{b=1\}\right].
\end{align}} 
This formulation may be preferable in edge cases where $Y=0$ or $Y=1$ almost surely. The proof still holds since $\mu_{[a,b]}\in[0,1]$. For simplicity, we proceed with the simpler formulation in the main text.
\end{proof}

\subsection{Proof of Theorem~\ref{thm:consistency-prediction}}\label{proof:thm:consistency-prediction}

\begin{proof}
Let 
\begin{align*}
    \epsilon_\stat = &~ 2\Big(\frac{2\log(3n/\delta)}{\lambda n}\Big)^{1/3}, \\
    \epsilon_1 = &~ \epsilon_\stat + \frac{4\epsilon}{\delta_1},\\
    K = &~ \max\left\{1, \left\lfloor\frac{1}{4\epsilon_1}\right\rfloor\right\}. 
\end{align*} 
If $K=1$, the interval $[0,1]$ is returned by the algorithm and covers all prediction values, proving the result directly with $C(n,\epsilon,\delta,\delta_1)=1$. For $K \geq 2$, the proof consists of three steps. 

\paragraph{Step 1. } 
We show that there exists $J \in \Z_+$ and disjoint intervals $[a_j,b_j], j \in [J]$ such that:
\begin{enumerate}
    \item For any $j \in [J]$, 
    \begin{align}\label{eq:consistency-existence-calibration}
        a_j + \epsilon_\stat \cdot \mathbbm{1}(a_j \neq 0) \leq \PP[Y=1\mid \estModel(X) \in [a_j,b_j]] \leq b_j - \epsilon_\stat \cdot \mathbbm{1}(b_j \neq 1);
    \end{align}
    \item 
    \begin{align}\label{eq:consistency-existence-coverage}
        \PP\Big(\estModel(X) \notin \cup_{j \in [J]} [a_j,b_j]\Big) \leq \delta_1;
    \end{align}
    \item For any $j \in [J]$,
    \begin{align}\label{eq:consistency-existence-length}
        4\epsilon_1 \leq \Big|b_j - a_j\Big| \leq 56\epsilon_1 \cdot \log(6/\delta_1).
    \end{align}
\end{enumerate}
First, applying Lemma~\ref{lem:dist-regularity} with $\nu$ being the distribution of $\estModel(X)$, uncovered mass budget $\delta_1$, and accuracy parameter $\epsilon_1$, we see that there exist disjoint intervals $[a_i,b_i], i \in [L]$ that satisfy 
\begin{align*}
    \PP\Big(\estModel(X) \notin \cup_{i \in [L]} [a_i,b_i]\Big) \leq \delta_1/2,
\end{align*}
Eq.~\eqref{eq:consistency-existence-length}, and
\begin{align}\label{eq:regularity-pm}
    a_i + \epsilon_1 \cdot \mathbbm{1}(a_i \neq 0) \leq \E[\estModel(X)\mid\estModel(X) \in [a_i,b_i]] \leq b_i - \epsilon_1 \cdot \mathbbm{1}(b_i \neq 1).
\end{align}
By Lemma~\ref{lem:conditional-dist-shift}, there exist $i_1,\dots,i_J \in [L]$ such that for any $j \in [J]$, 
\begin{align}\label{eq:conditional-dist-shift-pm}
    \Big|\PP(Y=1\mid\estModel(X) \in [a_{i_j},b_{i_j}])  - \E[\estModel(X)\mid\estModel(X) \in [a_{i_j},b_{i_j}]] \Big| \leq 4\epsilon/\delta_1 
\end{align}
and
\begin{align*}
    \PP\Big(\estModel(X) \in \cup_{j' \notin \{i_1,\dots,i_J\}} [a_{j'},b_{j'}]\Big) \leq \delta_1/2.
\end{align*}
Combining Eq.~\eqref{eq:conditional-dist-shift-pm} and Eq.~\eqref{eq:regularity-pm}, we have
\begin{align*}
    \MoveEqLeft{ \PP(Y = 1 \mid\estModel(X) \in [a_{i_j},b_{i_j}])} \\
    & \leq  \E[\estModel(X)\mid\estModel(X) \in [a_{i_j},b_{i_j}]] + 4\epsilon/\delta_1 \\
    & \leq  b_{i_j} - \epsilon_1 \cdot \mathbbm{1}(b_{i_j} \neq 1) + 4\epsilon/\delta_1 \\
    & =  b_{i_j} - \epsilon_\stat \cdot \mathbbm{1}(b_{i_j} \neq 1), 
\end{align*}
where the last step follows from the definition of $\epsilon_1$, and
\begin{align*}
    \MoveEqLeft{ \PP(Y = 1\mid\estModel(X) \in [a_{i_j},b_{i_j}])} \\
    & \geq  \E[\estModel(X)\mid\estModel(X) \in [a_{i_j},b_{i_j}]] - 4\epsilon/\delta_1\\
    & \geq   a_{i_j} + \epsilon_1 \cdot \mathbbm{1}(a_{i_j} \neq 0) - 4\epsilon/\delta_1\\
    & =  a_{i_j} + \epsilon_\stat \cdot \mathbbm{1}(a_{i_j} \neq 0) \tag{if $a_{i_j} = 0$ then trivially $a_{i_j} \leq \PP(Y=1\mid\estModel(X) \in [a_{i_j},b_{i_j}])$}.
\end{align*}
This establishes 
\begin{align*}
    a_{i_j} + \epsilon_\stat \cdot \mathbbm{1}(a_{i_j} \neq 0) \leq \PP(Y=1\mid\estModel(X) \in [a_{i_j},b_{i_j}]) \leq b_{i_j} - \epsilon_\stat \cdot \mathbbm{1}(b_{i_j} \neq 1),~\forall j \in [J].
\end{align*}
Furthermore,
\begin{align*}
    \PP\Big(\estModel(X) \notin \cup_{j \in [J]} [a_{i_j},b_{i_j}]\Big) \leq &~ \PP\Big(\estModel(X) \notin \cup_{i \in [L]} [a_i,b_i]\Big) + \PP\Big(\estModel(X) \in \cup_{j' \notin \{i_1,\dots,i_J\}} [a_{j'},b_{j'}]\Big) \\
    \leq &~ \delta_1.
\end{align*}
In sum, the intervals $[a_{i_1},b_{i_1}],\dots,[a_{i_J},b_{i_J}]$ satisfy Eq.~\eqref{eq:consistency-existence-calibration}-\eqref{eq:consistency-existence-length} (where we view $[a_{i_j},b_{i_j}]$ as $[a_{j},b_{j}]$).

\paragraph{Step 2. } 
We show that with probability at least $1-\delta$ over the randomness of $\{(x_i,y_i)\}_{i=1}^n$, the intervals in Step 1 will be included in the set of intervals $\I$ output by Algorithm~\ref{alg:calibration_population}. 
For simplicity of notation, we let $[a_{1},b_{1}],\dots,[a_J,b_J]$ denote the intervals found in Step 1; i.e. $[a_{1},b_{1}],\dots,[a_J,b_J] = [a_{i_1},b_{i_1}],\dots,[a_{i_J},b_{i_J}]$. Let $z_i = p(x_{m+i})$ be a shorthand for the predicted probability.
Define events
\begin{align*}
    A_1 = &~ \Big\{\sumin \mathbbm{1}(z_i \in [a_j,b_j]) \geq 2\lambda \epsilon_1 n,\forall j \in [J]\Big\}\\
    A_2 = &~ \Big\{\Big|\frac{\sum_{i=1}^n y_i \cdot \mathbbm{1}(z_i \in [a_j,b_j])}{\sum_{i=1}^n \mathbbm{1}(z_i \in [a_j,b_j])} - \PP(Y=1\mid \estModel(X) \in [a_j,b_j])\Big| \leq \epsilon_\stat/2, \forall j \in [J]\Big\}.
\end{align*}
Next, we show that under events $A_1, A_2$, the intervals $[a_{1},b_{1}],\dots,[a_J,b_J]$ are included in $\I$ by Algorithm~\ref{alg:calibration_population}. 
Indeed, we have
\begin{align*}
    \frac{\sum_{i=1}^n y_i \cdot \mathbbm{1}(z_i \in [a_j,b_j])}{\sum_{i=1}^n \mathbbm{1}(z_i \in [a_j,b_j])} \leq &~ \PP(Y=1\mid\estModel(X) \in [a_j,b_j]) + \epsilon_\stat/2\\
    \leq &~ b_{j} - \epsilon_\stat \cdot \mathbbm{1}(b_{j} \neq 1) + \epsilon_\stat/2\\
    \leq &~  b_{j} - \epsilon_\stat/2 \cdot \mathbbm{1}(b_{j} \neq 1)\\
    \leq &~ b_j - 2\sqrt{\frac{\log (2K^2/\delta)}{\sum_{i=1}^n \mathbbm{1}(z_i \in [a_j,b_j])}},
\end{align*}
where the last inequality is due to
\begin{align*}
    2\sqrt{\frac{\log (2K^2/\delta)}{\sum_{i=1}^n \mathbbm{1}(z_i \in [a_j,b_j])}} \leq &~ 2\sqrt{\frac{2\log(3n/\delta)}{2\lambda \epsilon_1 n}}\\
    \leq &~ \Big(\frac{2\log(3n/\delta)}{\lambda n}\Big)^{1/3} = \epsilon_\stat/2,
\end{align*}
under event $A_1$. 
Similarly,
\begin{align*}
    \frac{\sum_{i=1}^n y_i \cdot \mathbbm{1}(z_i \in [a_j,b_j])}{\sum_{i=1}^n \mathbbm{1}(z_i \in [a_j,b_j])} \geq a_j + 2\sqrt{\frac{\log (2K^2/\delta)}{\sum_{i=1}^n \mathbbm{1}(z_i \in [a_j,b_j])}}.
\end{align*}
Therefore, the interval $[a_j,b_j]$ will be included in $\I$ under $A_1 \cap A_2$. 

Next, we show $\PP(A_1\cap A_2)\ge 1-\delta$. Since 
$
    \Big|b_j - a_j\Big| \geq 4\epsilon_1,
$ 
applying Lemma~\ref{lem:sample-size-interval} with $c=4\epsilon_1$, 
we have 
\begin{align*}
    \PP(A_1) \geq &~ 1-J\exp\Big(-\frac{\lambda\epsilon_1 n}{2}\Big)\\
    \geq &~ 1- n\exp\Big(-\frac{\lambda\epsilon_\stat n}{2}\Big)\\
    \geq &~ 1- n\exp\Big(-\log(3n/\delta)\Big) \tag{$n \geq \lambda^{-1} \log(3n/\delta) / \sqrt{2}$}\\
    \geq &~ 1-\delta/3 .
\end{align*}

Under $A_1$, by Hoeffding's inequality and a union bound, we have
\begin{align*}
    \PP(A_2\mid A_1) \geq &~ 1- 2 J\exp\Big(-\epsilon_\stat^2 \cdot\sum_{i=1}^n \mathbbm{1}(z_i \in [a_j,b_j])/2\Big)\\
    \geq &~ 1- 2 n\exp\Big(-\epsilon_\stat^2 \cdot \lambda \epsilon_1 n\Big)\\
    \geq &~ 1- 2 n\exp\Big(-\epsilon_\stat^3 \cdot \lambda  n\Big)\\
    \geq &~ 1- 2 n\exp\Big(- \frac{16\log(3n/\delta)}{\lambda n } \cdot \lambda  n\Big)\\
    \geq &~ 1- 2 n\exp\Big(- 16\log(3n/\delta)\Big)\\
    \geq &~ 1- 2 n (3n/\delta)^{- 16}\\
    \geq &~ 1- \frac{2 \delta^{16}}{3^{16} n^{15}}\\
    \geq &~ 1-\delta/6.
\end{align*}
Combining, we have
\begin{align*}
    \PP(A_1 \cap A_2) \geq 1 - \delta.
\end{align*}

\paragraph{Step 3. } We show that the intervals $[a_{1},b_{1}],\dots,[a_J,b_J]$ in Step 2 satisfy the property
\begin{align*}
    \PP_{x \sim \mu}\Big(\exists [a_j,b_j]  ~s.t.~ |b_j-a_j| \leq C(n,\epsilon,\delta,\delta_1) ~\text{and}~ p(x) \in [a_j,b_j] \Big) \geq 1-\delta_1,
\end{align*}
where 
\begin{align*}
    C(n,\epsilon,\delta,\delta_1) &= \log(6/\delta_1) \cdot 56\Big(2\Big(\frac{2\log(3n/\delta)}{\lambda n}\Big)^{1/3}  + \frac{4\epsilon}{\delta_1}\Big).
\end{align*}
Define event
\begin{align*}
    A_3 = \Big\{\exists [a_j,b_j]  ~s.t.~ |b-a| \leq C(n,\epsilon,\delta,\delta_1) ~\text{and}~ p(x) \in [a_j,b_j] \Big\}.
\end{align*}
By Eq.~\eqref{eq:consistency-existence-length}, for any $j \in [J]$,
\begin{align*}
    |b_j-a_j| \leq C(n,\epsilon,\delta,\delta_1).
\end{align*}
It follows from Eq.~\eqref{eq:consistency-existence-coverage} that
\begin{align*}
    \PP_{x \sim \mu}\Big(\exists [a_j,b_j]  ~s.t.~ |b_j-a_j| \leq C(n,\epsilon,\delta,\delta_1) ~\text{and}~ p(x) \in [a_j,b_j] \Big) = &~ \PP\Big(\estModel(X) \in \cup_{j \in [J]} [a_j,b_j]\Big) \\
    \geq &~ 1- \delta_1.
\end{align*}
\end{proof}

\subsection{Proof of Lemma~\ref{lem:sample-size-interval}}\label{proof_lem:sample-size-interval}
\begin{proof}
Fix an interval $[a_j, b_j]$ for some $j \in [J]$.
Define $w_i := \mathbbm{1}(z_i \in [a_j,b_j])$ for $i \in [n]$. Then $\{w_i\}_{i=1}^n$ are i.i.d.\ Bernoulli random variables with mean
\begin{align*}
    \E[w_i] 
    = \PP(z_i \in [a_j,b_j])
    \ge \lambda (b_j-a_j)
    \ge \lambda c. 
\end{align*}
Let $S_j := \sumin w_i$. The total expected sample count satisfies $\E[S_j] \ge \lambda cn$. 
By the lower-tail multiplicative Chernoff bound,
$$\mathbb{P}(S_j \le (1-\delta)\E[S_j]) \le \exp(-\delta^2 \E[S_j]/2).$$ 
Applying it with $\delta = 1/2$,  
\begin{align*}
    \PP\Big(\sumin w_i < \lambda cn/2\Big) 
    \le &~ \PP\Big(S_j \le \frac{\E[S_j]}{2}\Big) \\
    \le &~ \exp\Big(-\frac{\E[S_j]}{8}\Big) \\
    \le &~ \exp\Big(-\frac{\lambda cn}{8}\Big).
\end{align*}
By a union bound over all $j \in [J]$, with probability at least $1-J\exp\Big(-\frac{\lambda cn}{8}\Big)$,
\begin{align*}
    \sumin \mathbbm{1}(z_i \in [a_j,b_j]) \geq \lambda cn/2
\end{align*}
holds for all $j \in [J]$. This completes the proof.
\end{proof}

\subsection{Proof of Lemma~\ref{lem:dist-regularity}}\label{proof_lem:dist-regularity}

\begin{proof}
We present a constructive proof based on discretization of $[0,1]$. To simplify notation, we abbreviate $\E_{Z \sim \nu}$ as $\E$ and $\PP_{Z \sim \nu}$ as $\PP$.

\paragraph{Step 1: Construct the graph.}
Consider $[(i-1)/\discreteNum, i/\discreteNum], i = 1,2, \dots,\discreteNum$ with $\discreteNum = \lfloor 1/(4\epsilon) \rfloor$. Without loss of generality assume $\discreteNum\ge2$.
Construct a directed graph $G = (V,E)$ where $V = [\discreteNum]$ and let the edges be defined as follows: $1 \ra 2, \discreteNum-1 \la \discreteNum$, where there is a directed edge $i \ra j$ iff both of the following conditions hold:
\begin{itemize}
    \item $|i-j| = 1$, and
    \item Either \begin{align}\label{eq:edge-condition-1}
   \PP(Z\in [(i-1)/\discreteNum, i/\discreteNum]) \text{ and } \Big(\E\Big[Z\mid Z \in [(i-1)/\discreteNum, i/\discreteNum]\Big] - \frac{i-0.5}{\discreteNum}\Big) \cdot (j-i) \ge 0 ,
\end{align}
or 
\begin{align}\label{eq:edge-condition-2}
    \PP\Big(Z \in [(i-1)/\discreteNum, i/\discreteNum]\Big) \leq 3 \cdot \PP\Big(Z \in [(j-1)/\discreteNum, j/\discreteNum]\Big).
\end{align}
\end{itemize}

Eq.~\eqref{eq:edge-condition-1} can be interpreted as: if $\E\Big[Z\mid Z \in [(i-1)/\discreteNum, i/\discreteNum]\Big]$ is closer to $(i-1)/\discreteNum$, then draw a directed edge $i-1 \la i$; if $\E\Big[Z\mid Z \in [(i-1)/\discreteNum, i/\discreteNum]\Big]$ is closer to $i/\discreteNum$, then draw a directed edge $i \ra i+1$. It directs an edge toward the side of the bin on which its conditional mean lies.

Eq.~\eqref{eq:edge-condition-2} can be interpreted as: we draw a directed edge to $i \ra j$ (for $j = i-1$ or $i+1$) if the probability mass on $[(i-1)/\discreteNum, i/\discreteNum]$ is not substantially greater than that of $[(j-1)/\discreteNum, j/\discreteNum]$. It directs an edge toward a neighboring bin whose mass is not substantially smaller.

By construction, it is obvious that every vertex has at least one outgoing edge since for positive-mass vertex Eq.~\eqref{eq:edge-condition-1} holds either for $j = i-1$ or $j=i+1$, and for zero-mass vertex Eq.~\eqref{eq:edge-condition-2} holds for both neighbors.

There is at least one bidirectional adjacent pair. The existence of such pair follows from a continuity argument: since 
$1 \ra 2, \discreteNum-1 \la \discreteNum$, and every vertex has at least one outgoing edge, let $k^*\in \{2,\dots,\discreteNum\}$ be the smallest integer such that there exists an edge $k^* -1 \la k^*$. Then $k^* -1 \ra k^*$ because every vertex has at least one outgoing edge, and if the outgoing edge for $k^*-1$ is $k^*-1 \rightarrow k^*-2$, then  $k^*$ would not be the smallest vertex as defined. This means that $k^* - 1 \leftrightarrows k^*$ forms a bidirectional pair, thus confirming the existence.

Now, denote the collection of $H$ maximal blocks of consecutive vertices connected by bidirectional edges as
\begin{align*}
    \{D_i: D_i = \{l_i, l_i+1, \dots, r_i\} , \text{ where }l_i\leftrightarrows l_i+ 1\leftrightarrows\ldots\leftrightarrows r_i, ~i\in[H] \}
\end{align*}
Then we have 
\begin{align*}
    1\le l_1 < r_1 < l_2 < r_2 < \ldots < l_H \le \discreteNum.
\end{align*}

\paragraph{Step 2: Construct intervals from the graph.}
We now construct intervals, where we first subdivide each block into short groups, and then assign the vertices between consecutive blocks to one of the two neighboring groups.

We now divide each bidirectional block $D_i$ into $m_i:=\lfloor(r_i- l_i+1)/2\rfloor$ consecutive groups, whereas every group has two vertices, except that the last has three when
$r_i- l_i+1$ is odd. 
That is, denote the $j$-th group as
$\{\alpha_{i,j},\dots,\beta_{i,j}\}$, with endpoints
\begin{align}\label{eq:block-core-endpoints}
 \alpha_{i,j}&= l_i+2(j-1),\nonumber\\
 \beta_{i,j}&=
 \begin{cases}
   l_i+2j-1,&j<m_i,\\
  r_i,&j=m_i.
 \end{cases}
\end{align}

Notice that for any $i\in[H-1]$, all integer vertices between two consecutive bidirectional blocks must be connected separately by a leftwards chain and a rightwards chain (upon merging duplicated edges) as one move from left to right. That is,
\begin{align*}
    r_i \la \cdots \la w_i \ra \cdots \ra l_{i+1},
\end{align*}
for some $w_i
\in \{r_i, \dots, l_{i+1}\}$. 
(Otherwise, there exists another bidirectional block between $r_i$ and $l_{i+1}$.) We call $w_i$ as an intervening source vertex, where we assign it the block on its left when constructing intervals. Denote the endpoints of constructed intervals as $\{u_i, v_i\}, i\in [H]$.
Define the cut points and the ranges assigned to the blocks by
\begin{align}\label{eq:block-assignment}
 \omega_0&=0,\qquad\omega_H=\discreteNum,\qquad
 \omega_i=\min\{w_i, l_{i+1}-1\}\quad(i\in[H-1]),\nonumber\\
 u_i&=\omega_{i-1}+1,\qquad v_i=\omega_i\quad(i\in[H]).
\end{align}
They satisfy
\begin{align*}
 u_i\leq l_i<r_i\leq v_i,
 \qquad u_{i+1}=v_i+1\quad(i\in[H-1]),
\end{align*}
and all edges in each assigned range have the form
\begin{align}\label{eq:block-with-chains}
 u_i\ra\cdots\ra l_i\leftrightarrows\cdots
 \leftrightarrows r_i\la\cdots\la v_i.
\end{align}

Truncate the chains in
Eq.~\eqref{eq:block-with-chains} after at most $\lceil \log(3/\delta)\rceil$ vertices by setting
\begin{align}\label{eq:truncated-block-endpoints}
 s_i=u_i\vee( l_i-\lceil \log(3/\delta)\rceil),\qquad
 t_i=v_i\wedge(r_i+\lceil \log(3/\delta)\rceil).
\end{align}
Attach the retained left chain to the first group and the retained
right chain to the last group. That is, set
\begin{align*}
 s_{i,j}&=
 \begin{cases}
  s_i,&j=1,\\
  \alpha_{i,j},&j>1,
 \end{cases}
 &
 t_{i,j}&=
 \begin{cases}
  \beta_{i,j},&j<m_i,\\
  t_i,&j=m_i.
 \end{cases}
\end{align*}
Define intervals
\begin{align}\label{eq:constructed-intervals}
 I_{i,j}=[a_{i,j},b_{i,j}]
 =[(s_{i,j}-1)/\discreteNum,t_{i,j}/\discreteNum],
 \qquad i\in[H],\ j\in[m_i].
\end{align}
An example illustration of the intervals (in boxes, without truncation) as: 
{\scriptsize
\begin{align}\label{eq:endpoint_illustration}
  \cdots \la v_{i-1} \ra \underbrace{\boxed{ u_i\ra\cdots\ra l_i\leftrightarrows l_i + 1 } }_{\text{attach chain to first block pair}}\leftrightarrows \underbrace{\boxed{ l_i + 2 \leftrightarrows l_i + 3} }_{\text{grouped by 2}} \leftrightarrows   \cdots 
 \leftrightarrows \underbrace{\boxed{ r_i - 1 \leftrightarrows  r_i\la\cdots\la v_i}}_{\text{attach chain to last block pair}}  \ra \cdots .
\end{align}
}
Truncation could be illustrated as:
{\scriptsize
\begin{align}\label{eq:truncation_illustration}
 \cdots \la v_{i-1} \ra  u_i\ra\cdots\ra \underbrace{\boxed{l_i - \lceil \log(3/\delta)\rceil\ra \cdots\ra l_i\leftrightarrows l_i + 1 } }_{\text{attach truncated chain}}\leftrightarrows \underbrace{\boxed{ l_i + 2 \leftrightarrows l_i + 3} }_{\text{grouped by 2}} \leftrightarrows   \cdots 
 .
\end{align}

}

Remove any intervals of zero probability. Relabel those remaining as
$I_i=[a_i,b_i]$, $i\in[L]$, in increasing order. 


\paragraph{Step 3: Establish theoretical properties.}

We show that the intervals $[a_i,b_i], i\in[L]$ satisfy the desired properties. 
By construction each interval contains between $1$ and $2\lceil \log(3/\delta)\rceil+3$ vertices, thus Eq.~\eqref{eq:regular-condition-length} is satisfied.
To show Eq.~\eqref{eq:regular-condition}, Fix $i \in [L]$. 
By $s_i \geq u_i, t_i \leq v_i$ and the definition of $u_i,v_i$ (after relabeling), there exist edges $s_i \ra s_i + 1$ and $t_i - 1 \la t_i$. We will show that $s_i \ra s_i + 1$ implies that $\E[Z\mid Z \in [a_i,b_i]] \geq a_i + \epsilon \cdot \mathbbm{1}(a_i \neq 0)$ by considering the following three cases of $s_i \ra s_i + 1$:

\paragraph{Case 1: $s_i = 1$. } Then $a_i=0$ and 
trivially
$
    \E[Z\mid Z \in [a_i,b_i]] \geq a_i + \epsilon \cdot \mathbbm{1}(a_i \neq 0).
$

\paragraph{Case 2:  }
The edge is drawn by Eq.~\eqref{eq:edge-condition-1}, which is
\begin{align*}
    \E\Big[Z\mid Z \in [(s_i-1)/\discreteNum, s_i/\discreteNum]\Big] \geq  \frac{s_i-0.5}{\discreteNum} .
\end{align*}
Then we have
\begin{align*}
    \E[Z\mid Z \in [a_i,b_i]] = &~ \sum_{j=s_i}^{t_i} \frac{\PP(Z \in [(j-1)/\discreteNum, j/\discreteNum])}{\PP(Z \in [a_i,b_i])} \cdot \E[Z\mid Z \in [(j-1)/\discreteNum, j/\discreteNum]]\\
    \geq &~ \E\Big[Z\mid Z \in [(s_i-1)/\discreteNum, s_i/\discreteNum]\Big] \\
    \geq &~ \frac{s_i-0.5}{\discreteNum} \\
    \geq &~ a_i + \epsilon.
\end{align*}
The first step comes from the law of total expectation and the fact that $[(j-1)/\discreteNum, j/\discreteNum], j=s_i,\dots, t_i$ forms a partition of $[a_i,b_i]$; the second step comes from $\E[Z\mid Z \in [(j-1)/\discreteNum, j/\discreteNum]] \geq \E\Big[Z\mid Z \in [(s_i-1)/\discreteNum, s_i/\discreteNum]\Big]$ for all $j \geq s_i$; the third step comes from Eq.~\eqref{eq:edge-condition-1}; and the final step comes from
$a_i = (s_i-1)/\discreteNum,$ $ \epsilon \le 1/(4\discreteNum)$, so $ a_i + \epsilon \le (s_i-3/4)/\discreteNum \le (s_i - 0.5)/\discreteNum$.

\paragraph{Case 3: } 
The edge is drawn by Eq.~\eqref{eq:edge-condition-2}, which is
\begin{align*}
    \PP\Big(Z \in [(s_i-1)/\discreteNum, s_i/\discreteNum]\Big) \leq 3 \cdot \PP\Big(Z \in [s_i/\discreteNum, (s_i+1)/\discreteNum]\Big).
\end{align*}
Under Case 3, we have
\begin{align*}
    \E[Z\mid Z \in [a_i,b_i]] = &~ \sum_{j=s_i}^{t_i} \frac{\PP(Z \in [(j-1)/\discreteNum, j/\discreteNum])}{\PP(Z \in [a_i,b_i])} \cdot \E[Z \mid Z \in [(j-1)/\discreteNum, j/\discreteNum]]\\
    \geq &~ \frac{\PP(Z \in [(s_i-1)/\discreteNum, s_i/\discreteNum])}{\PP(Z \in [a_i,b_i])} \cdot \E[Z \mid Z \in [(s_i-1)/\discreteNum, s_i/\discreteNum]] \\
    &~ + \left(1-\frac{\PP(Z \in [(s_i-1)/\discreteNum, s_i/\discreteNum])}{\PP(Z \in [a_i,b_i])}\right) \cdot \E[Z \mid Z \in [s_i/\discreteNum, (s_i+1)/\discreteNum]] \\
    \geq &~ 3/4 \cdot (s_i-1)/\discreteNum + 1/4 \cdot s_i/\discreteNum\\
    \ge &~ a_i + \epsilon. 
\end{align*}

The first step comes from the law of total expectation and $[(j-1)/\discreteNum, j/\discreteNum], j=s_i,\dots, t_i$ forms a partition of $[a_i,b_i]$; the second step comes from pulling the first term out $j=s_i$ and then using $\E[Z\mid Z \in [(j-1)/\discreteNum, j/\discreteNum]] \geq \E\Big[Z\mid Z \in [s_i/\discreteNum, (s_i+1)/\discreteNum]\Big]$ for all $j \geq s_i+1$; the third step comes from Eq.~\eqref{eq:edge-condition-2}; the last step comes from $a_i = (s_i-1)/\discreteNum,$ $ \epsilon \le 1/(4\discreteNum)$, so $ a_i + \epsilon \le (s_i-3/4)/\discreteNum$.

Repeating similar arguments we can show $t_i - 1 \la t_i $ implies that $\E[Z\mid Z \in [a_i,b_i]] \leq b_i - \epsilon \cdot \mathbbm{1}(b_i \neq 1)$. Therefore, Eq.~\eqref{eq:regular-condition} is confirmed.

Finally, we establish Eq.~\eqref{eq:regular-condition-coverage} by bounding the probability of the cells removed when the
chains were truncated. 
Define $q_i:=\PP( Z \in [(i-1)/\discreteNum, i/\discreteNum])$. Recall that as demonstrated, on the left chain assigned to block $D_i$,
the reverse edge is absent at every step. In particular, for each
$d\in\{1,\dots, l_i-u_i\}$, the edge
$ l_i-d+1\ra l_i-d$ does not exist and thus Eq.\eqref{eq:edge-condition-2} is violated. Therefore, 
\begin{align*}
 q_{ l_i-d+1}>3q_{ l_i-d}.
\end{align*}
Iterating this inequality along the chain gives
\begin{align*}
 q_{ l_i-d}<3^{-d}q_{ l_i}.
\end{align*}
We then have 
\begin{align*}    \sum_{d=\lceil\log(3/\delta)\rceil+1}^{ l_i-u_i}q_{ l_i-d} < \sum_{d=\lceil\log(3/\delta)\rceil+1}^{ l_i-u_i}3^{-d} q_{ l_i} .
\end{align*}
The analogous argument on the right chain ($d\in\{1,\dots,v_i-r_i\}$) gives
\begin{align*}
 q_{r_i+d}&<3^{-d}q_{r_i}
\\ \sum_{d=\lceil\log(3/\delta)\rceil+1}^{v_i-r_i}q_{r_i+d} &< \sum_{d=\lceil\log(3/\delta)\rceil+1}^{v_i-r_i}3^{-d}q_{r_i}.
\end{align*}

Combining, the total mass sum of discarded cells on two sides
is at most
\begin{align*}
 \sum_{d=\lceil\log(3/\delta)\rceil+1}^{ l_i-u_i}q_{ l_i-d}
 +\sum_{d=\lceil\log(3/\delta)\rceil+1}^{v_i-r_i}q_{r_i+d} &\leq (q_{ l_i}+q_{r_i})\sum_{d=\lceil\log(3/\delta)\rceil+1}^{\infty}3^{-d}\\
 &=\frac12\,3^{-\lceil\log(3/\delta)\rceil}(q_{ l_i}+q_{r_i}) \tag{sum of geometric series}.
\end{align*}

 Recall that each point can at most belong to two of the intervals,
\begin{align*}
 \sum_{i=1}^H(q_{ l_i}+q_{r_i})\leq2,
\end{align*}
by union bound,
\begin{align*}
 \PP\Big(Z\notin\cup_{i\in[L]}[a_i,b_i]\Big)
 &\leq\frac12\,3^{-\lceil\log(3/\delta)\rceil}\sum_{i=1}^H(q_{ l_i}+q_{r_i})\\
 &\leq3^{-\lceil\log(3/\delta)\rceil}\leq e^{-\lceil\log(3/\delta)\rceil}\leq\delta/3\leq\delta.
\end{align*}

This proves Eq.~\eqref{eq:regular-condition-coverage} and completes
the proof.

\end{proof}

\subsection{Proof of Lemma~\ref{lem:conditional-dist-shift}}
\label{proof_lem:conditional-dist-shift}

\begin{proof}
Let $j_1,j_2,\dots,j_M$ be all integers in $[L]$ such that Eq.~\eqref{eq:distribution-shift} is not satisfied. If $M=0$, then all intervals satisfy Eq.~\eqref{eq:distribution-shift}, and choosing $\{i_1,\dots,i_J\} = [L]$ completes the proof.

Otherwise, for any $w \in \{j_1,j_2,\dots,j_M\}$, we have $\PP(V \in [a_{w},b_{w}]) > 0$ and
\begin{align*}
    \frac{2\epsilon}{\delta} < &~ \Big|\E\Big[U\mid V \in [a_{w},b_{w}]\Big] - \E\Big[V\mid V \in [a_{w},b_{w}]\Big]\Big|\\
    \leq &~ \E\Big[|U-V| \mid V \in [a_{w},b_{w}]\Big] \tag{by the triangle inequality}\\
    = &~ \frac{1}{\PP(V \in [a_{w},b_{w}])} \cdot \E\Big[|U-V| \cdot \mathbbm{1}(V \in [a_{w},b_{w}])\Big].
\end{align*}
Each point belongs to at most two intervals, so  $\sum_{w \in \{j_1,\dots,j_M\}} \mathbbm{1}(V \in [a_w,b_w]) \leq 2$. It follows that
\begin{align*}
    2\epsilon \geq &~ 2\E[|U-V|]\\
    \geq &~ \E\left[|U-V| \sum_{w \in \{j_1,\dots,j_M\}} \mathbbm{1}(V \in [a_{w},b_{w}])\right]\\
    = &~ \sum_{w \in \{j_1,j_2,\dots,j_M\}}\E\Big[|U-V| \cdot \mathbbm{1}(V \in [a_{w},b_{w}])\Big]\\
    > &~ \sum_{w \in \{j_1,j_2,\dots,j_M\}}\PP(V \in [a_{w},b_{w}]) \cdot \frac{2\epsilon}{\delta}\\
    \geq &~ \PP\Big(V \in \cup_{w \in \{j_1,j_2,\dots,j_M\}} [a_w,b_w]\Big) \cdot \frac{2\epsilon}{\delta}.
\end{align*}
Dividing both sides by $2\epsilon/\delta$ yields
\begin{align*}
    \PP\Big(V \in \cup_{w \in \{j_1,j_2,\dots,j_M\}} [a_w,b_w]\Big) < \delta.
\end{align*}
Setting $\{i_1,\dots,i_J\} := [L] \setminus \{j_1,\dots,j_M\}$, the remaining intervals satisfy Eq.~\eqref{eq:distribution-shift} by construction and Eq.~\eqref{eq:distribution-shift-coverage}.
\end{proof}

\subsection{Proof of Theorem~\ref{thm:tp}}
\label{proof_thm:tp}
\begin{proof}
Let $A_j=\{p(X)\in\interval_j\}$, $j\in[\numInterval]$, and
$A_0=\{p(X)\notin\cup_{j=1}^{\numInterval}\interval_j\}$.
These events form a partition. For each $j$ with $m_j>0$, Eq.~\eqref{eq:task-predicting_true_prob} gives $C_{\tp}(X;\epsilon)=E^\epsilon(\interval_j)$ on $A_j$.
The law of total probability therefore implies
\begin{align}\label{eq:task-predicting_true_prob_decomposition}
 \PP\Big(p^*(X)\notin C_{\tp}(X;\epsilon)\Big)~\leq\sum_{j:m_j>0}m_j
 \PP\Big(p^*(X)\notin E^\epsilon(\interval_j)\mid A_j\Big)
 +\PP(A_0).
\end{align}
For a positive-mass interval, let $\mu_j=\E[p^*(X)\mid A_j]$.
By the tower property and the \titleShort condition,
\begin{align*}
 \mu_j=\E[\E[Y\mid X]\mid A_j]=\E[Y\mid A_j]\in\interval_j.
\end{align*}
If $p^*(X)\notin E^\epsilon(\interval_j)$, its distance from every
point of $\interval_j$, including $\mu_j$, exceeds $\epsilon$.
Chebyshev's inequality conditionally on $A_j$ gives
\begin{align*}
 \PP\Big(p^*(X)\notin E^\epsilon(\interval_j)\mid A_j\Big)
 &\leq\PP\Big(|p^*(X)-\mu_j|>\epsilon\mid A_j\Big)\\
 &\leq\frac{\Var(p^*(X)\mid A_j)}{\epsilon^2}.
\end{align*}
Since $\PP(A_0)=1-\sum_{j=1}^{\numInterval}m_j$, substituting them into Eq.~\eqref{eq:task-predicting_true_prob_decomposition} yields
\begin{align*}
 \PP\Big(p^*(X)\notin C_{\tp}(X;\epsilon)\Big)~\leq\epsilon^{-2}\sum_{j:m_j>0}m_j
 \Var\Big(p^*(X)\mid p(X)\in\interval_j\Big)
 +\Big(1-\sum_{j=1}^{\numInterval}m_j\Big).
\end{align*}
This proves the theorem.
\end{proof}
\subsection{Proof of Remark~\ref{cor:tp_smoothness}}
\label{proof_cor:tp_smoothness}

\begin{proof}
    Suppose that
$\|p^*-p\|_\infty\leq\eta$ for some $\eta\geq0$.
For every positive-mass interval $\interval_j=[a_j,b_j]$, conditioning
on $p(X)\in\interval_j$ gives
\begin{align*}
 p^*(X)\in[\max\{0,a_j-\eta\},\min\{1,b_j+\eta\}]
 \qquad\text{almost surely}.
\end{align*}
The length of this range is at most $|\interval_j|+2\eta$.
Since a random variable supported on an interval of length $d$ has
variance at most $d^2/4$ by Popoviciu's inequality, we have
\begin{align*}
 \PP\Big(p^*(X)\notin C_{\tp}(X;\epsilon)\Big) 
 &~\leq\frac{1}{4\epsilon^2}\sum_{j:m_j>0}
 m_j(|\interval_j|+2\eta)^2
 +\Big(1-\sum_{j=1}^{\numInterval}m_j\Big).
\end{align*}

Alternatively, suppose that
$|p^*(x)-p^*(x')|\leq L\cdot|p(x)-p(x')|$ for all $x,x'\in\X$.
When $p(x),p(x')\in\interval_j$, their difference is at most
$|\interval_j|$, so the conditional range of $p^*(X)$ has length
at most $L|\interval_j|$. Applying the same variance bound gives
\begin{align*}
 \PP\Big(p^*(X)\notin C_{\tp}(X;\epsilon)\Big)
 &~\leq\frac{L^2}{4\epsilon^2}\sum_{j:m_j>0}
 m_j|\interval_j|^2
 +\Big(1-\sum_{j=1}^{\numInterval}m_j\Big).
\end{align*}
\end{proof}

\subsection{Proof of Theorem~\ref{thm:mc_error}}
\label{proof_thm:mc_error}

\begin{proof}
Fix a class $\iteratorClass\in[\numClass]$ and abbreviate
\[
\pi:=\pi_\iteratorClass,\qquad
t:=t_\iteratorClass,\qquad
J:=J_\iteratorClass,\qquad
S:=S_\iteratorClass,
\]
\[
m_j:=m_{\iteratorClass,j},\qquad
a_j:=a_{\iteratorClass,j},\qquad
b_j:=b_{\iteratorClass,j},
\qquad j=1,\ldots,J.
\]
Define the disjoint events
\[
A_0
:=
\left\{
p_\iteratorClass(X)\notin \bigcup_{j=1}^{J}[a_j,b_j]
\right\},
\qquad
A_j
:=
\bigl\{
p_\iteratorClass(X)\in [a_j,b_j]
\bigr\},
\quad j=1,\ldots,J.
\]
Let
\[
\mu_0
:=
\PP\bigl(Y(X)=\iteratorClass \mid A_0\bigr)\in[0,1],
\qquad
\mu_j
:=
\PP\bigl(Y(X)=\iteratorClass \mid A_j\bigr)\in[a_j,b_j],
\quad j=1,\ldots,J,
\]
where the inclusion $\mu_j\in[a_j,b_j]$ follows from the definition of \titleShorts.
Since the events $A_0,A_1,\ldots,A_J$ partition the sample space,
\begin{align}
\pi
=
\PP\bigl(Y(X)=\iteratorClass\bigr)
=
t\mu_0+\sum_{j=1}^{J} m_j\mu_j.
\label{eq:mc_partition_identity}
\end{align}

Now,
\begin{align}
&~\PP\Bigl(Y(X)\notin C_{\mc}(X;S),\,Y(X)=\iteratorClass\Bigr)
\nonumber\\
=
&~t\mu_0+\sum_{j\notin S} m_j\mu_j,
\label{eq:mc_miscoverage_numerator}
\end{align}
because class $\iteratorClass$ is omitted precisely on the uncovered region $A_0$ and on those intervals $A_j$ with $j\notin S$.
Dividing Eq.~\eqref{eq:mc_miscoverage_numerator} by $\pi$ gives
\begin{align}
\PP\Bigl(Y(X)\notin C_{\mc}(X;S)\mid Y(X)=\iteratorClass\Bigr)
=
\frac{t\mu_0+\sum_{j\notin S} m_j\mu_j}{\pi}.
\label{eq:mc_conditional_error_exact}
\end{align}

We first prove the upper-endpoint bound.
Since $\mu_0\le 1$ and $\mu_j\le b_j$ for every $j\notin S$,
\[
t\mu_0+\sum_{j\notin S} m_j\mu_j
\le
t+\sum_{j\notin S} m_j b_j.
\]
Combining with Eq.~\eqref{eq:mc_conditional_error_exact} yields
\[
\PP\Bigl(Y(X)\notin C_{\mc}(X;S)\mid Y(X)=\iteratorClass\Bigr)
\le
\Gamma_{\iteratorClass}^{\mathrm{up}}(S_\iteratorClass).
\]

Next we prove the lower-endpoint bound.
By Eqs.~\eqref{eq:mc_partition_identity} and \eqref{eq:mc_miscoverage_numerator},
\[
t\mu_0+\sum_{j\notin S} m_j\mu_j
=
\pi-\sum_{j\in S} m_j\mu_j.
\]
Since $\mu_j\ge a_j$ for all $j\in S$,
\[
t\mu_0+\sum_{j\notin S} m_j\mu_j
\le
\pi-\sum_{j\in S} m_j a_j.
\]
Combining again with Eq.~\eqref{eq:mc_conditional_error_exact} gives
\[
\PP\Bigl(Y(X)\notin C_{\mc}(X;S)\mid Y(X)=\iteratorClass\Bigr)
\le
\Gamma_{\iteratorClass}^{\mathrm{low}}(S_\iteratorClass).
\]
Therefore
\[
\PP\Bigl(Y(X)\notin C_{\mc}(X;S)\mid Y(X)=\iteratorClass\Bigr)
\le
\Gamma_{\iteratorClass}(S_\iteratorClass).
\]

It remains to prove exactness.
Define
\[
U:=t+\sum_{j\notin S} m_j b_j,
\qquad
L:=\pi-\sum_{j\in S} m_j a_j.
\]
We show that the miscoverage numerator in Eq.~ \eqref{eq:mc_miscoverage_numerator} can attain $\min\{U,L\}$.

\emph{Case 1: $U\le L$.}
Set
\[
\mu_0=1,
\qquad
\mu_j=b_j \ \text{ for all } j\notin S.
\]
Then the miscoverage numerator equals $U$.
To make Eq.~\eqref{eq:mc_partition_identity} hold, we need to choose $\mu_j\in[a_j,b_j]$ for $j\in S$ so that
\[
\sum_{j\in S} m_j\mu_j=\pi-U.
\]
This is possible because
\[
U\le L=\pi-\sum_{j\in S} m_j a_j
\quad \Longrightarrow \quad
\pi-U\ge \sum_{j\in S} m_j a_j,
\]
and, since any admissible law must satisfy
\[
\pi
\le
t+\sum_{j\notin S} m_j b_j+\sum_{j\in S} m_j b_j
=
U+\sum_{j\in S} m_j b_j,
\]
we also have
\[
\pi-U\le \sum_{j\in S} m_j b_j.
\]
Hence $\pi-U$ lies in the interval
\[
\left[\sum_{j\in S} m_j a_j,\ \sum_{j\in S} m_j b_j\right],
\]
so such a choice of $\{\mu_j\}_{j\in S}$ exists.
Therefore equality is attained with value
\[
\frac{U}{\pi}
=
\Gamma_{\iteratorClass}^{\mathrm{up}}(S_\iteratorClass)
=
\Gamma_{\iteratorClass}(S_\iteratorClass).
\]

\emph{Case 2: $L<U$.}
Set
\[
\mu_j=a_j \ \text{ for all } j\in S.
\]
Then the retained intervals contribute exactly $\sum_{j\in S} m_j a_j$, so by Eq.~\eqref{eq:mc_partition_identity} the miscoverage numerator must equal $L$.
It remains to show that one can choose $\mu_0\in[0,1]$ and $\mu_j\in[a_j,b_j]$ for $j\notin S$ so that
\[
t\mu_0+\sum_{j\notin S} m_j\mu_j=L.
\]
The attainable range of the left-hand side is
\[
\left[\sum_{j\notin S} m_j a_j,\ t+\sum_{j\notin S} m_j b_j\right].
\]
Because any admissible law must satisfy
\[
\pi\ge \sum_{j=1}^{J} m_j a_j,
\]
we have
\[
L
=
\pi-\sum_{j\in S} m_j a_j
\ge
\sum_{j\notin S} m_j a_j.
\]
And by assumption,
\[
L<U=t+\sum_{j\notin S} m_j b_j.
\]
Hence $L$ lies in the attainable range, so equality is attained with value
\[
\frac{L}{\pi}
=
\Gamma_{\iteratorClass}^{\mathrm{low}}(S_\iteratorClass)
=
\Gamma_{\iteratorClass}(S_\iteratorClass).
\]

This proves exactness in both cases.
\end{proof}

\subsection{Proof of Theorem~\ref{thm:mc_optimality}}
\label{proof_thm:mc_optimality}

\begin{proof}
By construction, $\Gamma_{\iteratorClass}\bigl(S_\iteratorClass^\star(\alpha_\iteratorClass)\bigr)\le \alpha_\iteratorClass$ for every $\iteratorClass\in[\numClass]$.
Applying Theorem~\ref{thm:mc_error} with $S=S^\star(\alpha)$ yields
\[
\PP\Bigl(Y(X)\notin C_{\mc}^\star(X;\alpha)\mid Y(X)=\iteratorClass\Bigr)
\le
\alpha_\iteratorClass,
\]
which proves the coverage claim.

We now prove size-optimality.
For any $S=(S_1,\ldots,S_\numClass)$,
\begin{align*}
\E\bigl[|C_{\mc}(X;S)|\bigr]
=
\sum_{\iteratorClass=1}^{\numClass}
\PP\Bigl(\iteratorClass\in C_{\mc}(X;S)\Bigr).
\end{align*}
Since the intervals in $\intervalset_{p_\iteratorClass}$ are disjoint for each fixed class $\iteratorClass$,
\[
\PP\Bigl(\iteratorClass\in C_{\mc}(X;S)\Bigr)
=
\sum_{j\in S_\iteratorClass}
\PP\bigl(p_\iteratorClass(X)\in [a_{\iteratorClass,j},b_{\iteratorClass,j}]\bigr)
=
\sum_{j\in S_\iteratorClass} m_{\iteratorClass,j}.
\]
Therefore
\begin{align}
\E\bigl[|C_{\mc}(X;S)|\bigr]
=
\sum_{\iteratorClass=1}^{\numClass}
\sum_{j\in S_\iteratorClass} m_{\iteratorClass,j}.
\label{eq:mc_expected_size}
\end{align}

The objective \eqref{eq:mc_expected_size} separates over classes, and the feasibility constraints
\[
\Gamma_{\iteratorClass}(S_\iteratorClass)\le \alpha_\iteratorClass,
\qquad
\iteratorClass\in[\numClass],
\]
also separate over classes.
Hence the global minimization problem decomposes into $\numClass$ independent classwise problems, and each classwise optimizer is exactly $S_\iteratorClass^\star(\alpha_\iteratorClass)$ by \eqref{eq:mc_population_opt}.
Summing the classwise minima proves the result.
\end{proof}

\subsection{Proof of Theorem~\ref{thm:mc_same_sample}}
\label{proof_thm:mc_same_sample}

\begin{proof}
Let $\mathcal{E}_{\mathrm{PI}}$ denote the event that all intervals returned by Algorithm~1 are genuine disjoint \titleShorts simultaneously.
By assumption,
\[
\PP(\mathcal{E}_{\mathrm{PI}})\ge 1-\delta_{\mathrm{PI}}.
\]

For each class $\iteratorClass\in[\numClass]$ and each candidate interval $I\in\mathcal{A}_{p_\iteratorClass}$, define
\[
Z_k^{(\iteratorClass,I)}
:=
\mathbbm{1}\bigl(p_\iteratorClass(X_k)\in I\bigr),
\qquad
k=1,\ldots,n_{\mathrm{cal}}.
\]
Conditioned on the data-generating distribution, these are i.i.d.\ Bernoulli random variables with mean
\[
m_{\iteratorClass}(I)
:=
\PP\bigl(p_\iteratorClass(X)\in I\bigr).
\]
Hence Hoeffding's inequality gives, for every fixed $(\iteratorClass,I)$,
\[
\PP\left(
\left|
\frac{1}{n_{\mathrm{cal}}}\sum_{k=1}^{n_{\mathrm{cal}}} Z_k^{(\iteratorClass,I)}
-
m_{\iteratorClass}(I)
\right|
>
\varepsilon_m
\right)
\le
2e^{-2n_{\mathrm{cal}}\varepsilon_m^2}.
\]
Applying a union bound over all candidate intervals in all classes yields
\[
\PP\left(
\sup_{\iteratorClass\in[\numClass]}
\sup_{I\in\mathcal{A}_{p_\iteratorClass}}
\left|
\frac{1}{n_{\mathrm{cal}}}\sum_{k=1}^{n_{\mathrm{cal}}} Z_k^{(\iteratorClass,I)}
-
m_{\iteratorClass}(I)
\right|
>
\varepsilon_m
\right)
\le
2N_{\mathrm{cand}}e^{-2n_{\mathrm{cal}}\varepsilon_m^2}
=
\delta_m.
\]
Therefore, with probability at least $1-\delta_m$, we have simultaneously for every class $\iteratorClass$ and every candidate interval $I\in \mathcal{A}_{p_\iteratorClass}$,
\[
\left|
\widehat m_{\iteratorClass}(I)-m_{\iteratorClass}(I)
\right|
\le
\varepsilon_m.
\]
Since each selected interval $\widehat I_{\iteratorClass,j}$ belongs to $\mathcal{A}_{p_\iteratorClass}$ by construction, it follows that
\[
m_{\iteratorClass,j}
:=
\PP\bigl(p_\iteratorClass(X)\in \widehat I_{\iteratorClass,j}\bigr)
\ge
\underline m_{\iteratorClass,j}
\qquad
\text{for all } \iteratorClass \text{ and } j.
\]
Let $\mathcal{E}_m$ denote this event.

Next, for each class $\iteratorClass$, the random variables $\mathbbm{1}(Y_k=\iteratorClass)$ are i.i.d.\ Bernoulli with mean $\pi_\iteratorClass$.
Another Hoeffding bound and a union bound over $\iteratorClass\in[\numClass]$ give
\[
\PP\left(
\sup_{\iteratorClass\in[\numClass]}
\bigl|
\widehat\pi_\iteratorClass-\pi_\iteratorClass
\bigr|
>
\varepsilon_\pi
\right)
\le
2\numClass\,e^{-2n_{\mathrm{cal}}\varepsilon_\pi^2}
=
\delta_\pi.
\]
Hence, with probability at least $1-\delta_\pi$, we simultaneously have
\[
\underline\pi_\iteratorClass
\le
\pi_\iteratorClass
\le
\overline\pi_\iteratorClass
\qquad
\text{for all }\iteratorClass\in[\numClass].
\]
Let $\mathcal{E}_\pi$ denote this event.

By the union bound,
\[
\PP\bigl(\mathcal{E}_{\mathrm{PI}}\cap \mathcal{E}_m\cap \mathcal{E}_\pi\bigr)
\ge
1-\delta_{\mathrm{PI}}-\delta_m-\delta_\pi.
\]
Fix a realization of $D_{\mathrm{cal}}$ in the event
\[
\mathcal{E}
:=
\mathcal{E}_{\mathrm{PI}}\cap \mathcal{E}_m\cap \mathcal{E}_\pi.
\]
Conditioned on this realization, the intervals $\widehat I_{\iteratorClass,j}=[\widehat a_{\iteratorClass,j},\widehat b_{\iteratorClass,j}]$ are now deterministic and genuine \titleShorts, so Theorem~\ref{thm:mc_error} applies to them.

Fix a class $\iteratorClass\in[\numClass]$ and a subset $S_\iteratorClass\subseteq[\widehat J_\iteratorClass]$.
By Theorem~\ref{thm:mc_error},
\begin{align}
&~
\PP\Bigl(
Y(X)\notin \widehat C_{\mc}(X;S)
\ \Big|\
Y(X)=\iteratorClass,\ D_{\mathrm{cal}}
\Bigr)
\nonumber\\
\le&~
\min\Biggl\{
\frac{
\widehat t_\iteratorClass
+\sum_{j\notin S_\iteratorClass}
\widehat m_{\iteratorClass,j}^{\mathrm{pop}}
\widehat b_{\iteratorClass,j}
}{
\pi_\iteratorClass
},
\ 
1-
\frac{
\sum_{j\in S_\iteratorClass}
\widehat m_{\iteratorClass,j}^{\mathrm{pop}}
\widehat a_{\iteratorClass,j}
}{
\pi_\iteratorClass
}
\Biggr\},
\label{eq:mc_empirical_start}
\end{align}
where
\[
\widehat m_{\iteratorClass,j}^{\mathrm{pop}}
:=
\PP\bigl(p_\iteratorClass(X)\in \widehat I_{\iteratorClass,j}\mid D_{\mathrm{cal}}\bigr),
\qquad
\widehat t_\iteratorClass
:=
1-\sum_{j=1}^{\widehat J_\iteratorClass}\widehat m_{\iteratorClass,j}^{\mathrm{pop}}.
\]
Using
\[
\widehat t_\iteratorClass
+
\sum_{j\notin S_\iteratorClass}
\widehat m_{\iteratorClass,j}^{\mathrm{pop}}
\widehat b_{\iteratorClass,j}
=
1-
\sum_{j\in S_\iteratorClass}\widehat m_{\iteratorClass,j}^{\mathrm{pop}}
-
\sum_{j\notin S_\iteratorClass}
(1-\widehat b_{\iteratorClass,j})
\widehat m_{\iteratorClass,j}^{\mathrm{pop}},
\]
together with $\widehat m_{\iteratorClass,j}^{\mathrm{pop}}\ge \underline m_{\iteratorClass,j}$ on $\mathcal{E}_m$, we obtain
\[
\widehat t_\iteratorClass
+
\sum_{j\notin S_\iteratorClass}
\widehat m_{\iteratorClass,j}^{\mathrm{pop}}
\widehat b_{\iteratorClass,j}
\le
1-
\sum_{j\in S_\iteratorClass}\underline m_{\iteratorClass,j}
-
\sum_{j\notin S_\iteratorClass}
(1-\widehat b_{\iteratorClass,j})
\underline m_{\iteratorClass,j}.
\]
Since $\pi_\iteratorClass\ge \underline\pi_\iteratorClass$ on $\mathcal{E}_\pi$, it follows that
\[
\frac{
\widehat t_\iteratorClass
+\sum_{j\notin S_\iteratorClass}
\widehat m_{\iteratorClass,j}^{\mathrm{pop}}
\widehat b_{\iteratorClass,j}
}{
\pi_\iteratorClass
}
\le
\widehat\Gamma_{\iteratorClass}^{\mathrm{up}}(S_\iteratorClass).
\]
Likewise, on $\mathcal{E}_m$ and $\mathcal{E}_\pi$,
\[
\sum_{j\in S_\iteratorClass}
\widehat m_{\iteratorClass,j}^{\mathrm{pop}}
\widehat a_{\iteratorClass,j}
\ge
\sum_{j\in S_\iteratorClass}
\underline m_{\iteratorClass,j}
\widehat a_{\iteratorClass,j},
\qquad
\pi_\iteratorClass\le \overline\pi_\iteratorClass,
\]
so
\[
1-
\frac{
\sum_{j\in S_\iteratorClass}
\widehat m_{\iteratorClass,j}^{\mathrm{pop}}
\widehat a_{\iteratorClass,j}
}{
\pi_\iteratorClass
}
\le
1-
\frac{
\sum_{j\in S_\iteratorClass}
\underline m_{\iteratorClass,j}
\widehat a_{\iteratorClass,j}
}{
\overline\pi_\iteratorClass
}
=
\widehat\Gamma_{\iteratorClass}^{\mathrm{low}}(S_\iteratorClass).
\]
Combining these inequalities with Ineq. \eqref{eq:mc_empirical_start} yields
\[
\PP\Bigl(
Y(X)\notin \widehat C_{\mc}(X;S)
\ \Big|\
Y(X)=\iteratorClass,\ D_{\mathrm{cal}}
\Bigr)
\le
\widehat\Gamma_{\iteratorClass}(S_\iteratorClass).
\]
Since the argument holds simultaneously for every class and every subset on the event $\mathcal{E}$, the theorem follows.

Finally, if $\widehat S_\iteratorClass^\star(\alpha_\iteratorClass)$ is feasible for every class, then
\[
\widehat\Gamma_{\iteratorClass}\bigl(\widehat S_\iteratorClass^\star(\alpha_\iteratorClass)\bigr)
\le
\alpha_\iteratorClass.
\]
Applying the displayed bound above with $S_\iteratorClass=\widehat S_\iteratorClass^\star(\alpha_\iteratorClass)$ gives
\[
\PP\Bigl(
Y(X)\notin \widehat C_{\mc}^\star(X;\alpha)
\ \Big|\
Y(X)=\iteratorClass,\ D_{\mathrm{cal}}
\Bigr)
\le
\alpha_\iteratorClass
\]
for every class $\iteratorClass\in[\numClass]$.
\end{proof}
\clearpage
\section{Omitted Details for Label Set Construction in Classification}
\label{app:mc_details}

This section provides the details omitted from
Section~\ref{sec:task_2}: the decomposition of
the population optimization, and the empirically feasible implementation.



\subsection{Decomposition of the optimization for interval subset selection and label set construction}
\label{app:mc_details_optimization}
Recall the optimization
\[
S_\iteratorClass^\star(\alpha_\iteratorClass)
\in
\arg\min_{S\subseteq[J_\iteratorClass]}
\sum_{j\in S}m_{\iteratorClass,j}
\quad
\text{subject to}
\quad
\Gamma_{\iteratorClass}(S)
\le
\alpha_\iteratorClass.
\]
Since
\[
\Gamma_{\iteratorClass}(S)
=
\min\Bigl\{
\Gamma_{\iteratorClass}^{\mathrm{up}}(S),
\Gamma_{\iteratorClass}^{\mathrm{low}}(S)
\Bigr\},
\]
the feasible set is the union of the feasible sets corresponding to the
two certificates.

The lower-side certificate yields
\begin{align}
S_\iteratorClass^{\mathrm{keep}}(\alpha_\iteratorClass)
\in
\arg\min_{S\subseteq[J_\iteratorClass]}
&\sum_{j\in S}m_{\iteratorClass,j}
\label{eq:mc_keep_problem}\\
\text{subject to}\quad
&
\sum_{j\in S}
a_{\iteratorClass,j}m_{\iteratorClass,j}
\ge
(1-\alpha_\iteratorClass)\pi_\iteratorClass.
\nonumber
\end{align}

The upper-side certificate yields
\begin{align}
S_\iteratorClass^{\mathrm{drop}}(\alpha_\iteratorClass)
\in
\arg\min_{S\subseteq[J_\iteratorClass]}
&\sum_{j\in S}m_{\iteratorClass,j}
\label{eq:mc_drop_problem}\\
\text{subject to}\quad
&
t_\iteratorClass
+
\sum_{j\notin S}
m_{\iteratorClass,j}b_{\iteratorClass,j}
\le
\alpha_\iteratorClass\pi_\iteratorClass.
\nonumber
\end{align}

Thus, one solves both problems and retains whichever feasible solution has
the smaller objective value. We demonstrate it in pseudocode in Algorithm~\ref{alg:mc_population_optimal}. 

\begin{algorithm}[t]
\caption{Constructing the population-optimal label set}
\label{alg:mc_population_optimal}
\begin{algorithmic}[1]
\Require
Disjoint \titleShorts
$\intervalset_{p_\iteratorClass}
=
\{[a_{\iteratorClass,j},b_{\iteratorClass,j}]\}_{j=1}^{J_\iteratorClass}$,
population interval masses $\{m_{\iteratorClass,j}\}$,
class priors $\{\pi_\iteratorClass\}$,
and targets
$\alpha=(\alpha_1,\ldots,\alpha_\numClass)$
\Ensure
Population-optimal label set $C_{\mc}^\star(\cdot;\alpha)$
\For{$\iteratorClass=1$ to $\numClass$}
    \State
    Compute
    $
    t_\iteratorClass
    =
    1-\sum_{j=1}^{J_\iteratorClass}
    m_{\iteratorClass,j}.
    $
    \State
    Solve the keep-side problem
    \eqref{eq:mc_keep_problem}; if infeasible, assign objective value
    $+\infty$.
    \State
    Solve the drop-side problem
    \eqref{eq:mc_drop_problem}; if infeasible, assign objective value
    $+\infty$.
    \State
    Let
    $S_\iteratorClass^\star(\alpha_\iteratorClass)$
    be the feasible solution with smaller objective value.
\EndFor
\State
\Return
\[
C_{\mc}^\star(x;\alpha)
=
\Bigl\{
\iteratorClass\in[\numClass]:
p_\iteratorClass(x)
\in
\bigcup_{j\in
S_\iteratorClass^\star(\alpha_\iteratorClass)}
[a_{\iteratorClass,j},b_{\iteratorClass,j}]
\Bigr\}.
\]
\end{algorithmic}
\end{algorithm}

\subsection{Empirically feasible implementation using the same calibration sample}
\label{app:mc_details_same_sample}

The population-optimal algorithm presented above is not directly implementable because the interval masses $m_{\iteratorClass,j}$ and the class priors $\pi_\iteratorClass$ are unknown.
To obtain a practical algorithm we thus replace them by one-sided confidence bounds computed from the same calibration set that is already used to construct the \titleShorts.
Let
\[
D_{\mathrm{cal}}=\bigl\{(X_k,Y_k)\bigr\}_{k=1}^{n_{\mathrm{cal}}}
\]
be the calibration set, and let
\[
\widehat{\intervalset}_{p_\iteratorClass}
=
\bigl\{
\widehat I_{\iteratorClass,j}
=
[\widehat a_{\iteratorClass,j},\widehat b_{\iteratorClass,j}]
\bigr\}_{j=1}^{\widehat J_\iteratorClass}
\]
be the resulting disjoint \titleShorts for class $\iteratorClass$.
Let $\mathcal{A}_{p_\iteratorClass}$ denote the finite candidate interval family searched by Algorithm~\ref{alg:mc_population_optimal} for class $\iteratorClass$, and define
\[
N_{\mathrm{cand}}
:=
\sum_{\iteratorClass=1}^{\numClass}
\bigl|\mathcal{A}_{p_\iteratorClass}\bigr|.
\]
For each selected interval, define the empirical mass
\[
\widehat m_{\iteratorClass,j}
:=
\frac{1}{n_{\mathrm{cal}}}
\sum_{k=1}^{n_{\mathrm{cal}}}
\mathbbm{1}\Bigl(
p_\iteratorClass(X_k)\in \widehat I_{\iteratorClass,j}
\Bigr),
\qquad
\widehat\pi_\iteratorClass
:=
\frac{1}{n_{\mathrm{cal}}}
\sum_{k=1}^{n_{\mathrm{cal}}}
\mathbbm{1}(Y_k=\iteratorClass).
\]
For confidence levels $\delta_m,\delta_\pi\in(0,1)$, let
\[
\varepsilon_m
:=
\sqrt{
\frac{\log(2N_{\mathrm{cand}}/\delta_m)}{2n_{\mathrm{cal}}}
},
\qquad
\varepsilon_\pi
:=
\sqrt{
\frac{\log(2\numClass/\delta_\pi)}{2n_{\mathrm{cal}}}
}.
\]
We then define
\[
\underline m_{\iteratorClass,j}
:=
\bigl(\widehat m_{\iteratorClass,j}-\varepsilon_m\bigr)_+,
\qquad
\underline\pi_\iteratorClass
:=
\bigl(\widehat\pi_\iteratorClass-\varepsilon_\pi\bigr)_+,
\qquad
\overline\pi_\iteratorClass
:=
\bigl(\widehat\pi_\iteratorClass+\varepsilon_\pi\bigr)\wedge 1.
\]
For any subset $S_\iteratorClass\subseteq [\widehat J_\iteratorClass]$, define
\[
\widehat\Gamma_{\iteratorClass}^{\mathrm{up}}(S_\iteratorClass)
:=
\begin{cases}
\displaystyle
\frac{
1-\sum_{j\in S_\iteratorClass}\underline m_{\iteratorClass,j}
-\sum_{j\notin S_\iteratorClass}(1-\widehat b_{\iteratorClass,j})\underline m_{\iteratorClass,j}
}{\underline\pi_\iteratorClass},
& \underline\pi_\iteratorClass>0,\\[3ex]
+\infty,
& \underline\pi_\iteratorClass=0,
\end{cases}
\]
\[
\widehat\Gamma_{\iteratorClass}^{\mathrm{low}}(S_\iteratorClass)
:=
1-
\frac{
\sum_{j\in S_\iteratorClass}\widehat a_{\iteratorClass,j}\underline m_{\iteratorClass,j}
}{\overline\pi_\iteratorClass},
\qquad
\widehat\Gamma_{\iteratorClass}(S_\iteratorClass)
:=
\min\Bigl\{
\widehat\Gamma_{\iteratorClass}^{\mathrm{up}}(S_\iteratorClass),
\widehat\Gamma_{\iteratorClass}^{\mathrm{low}}(S_\iteratorClass)
\Bigr\}.
\]
This yields the empirical relaxation
\[
\widehat S_\iteratorClass^\star(\alpha_\iteratorClass)
\in
\arg\min_{S\subseteq[\widehat J_\iteratorClass]}
\sum_{j\in S}\widehat m_{\iteratorClass,j}
\quad \text{subject to}\quad
\widehat\Gamma_{\iteratorClass}(S)\le \alpha_\iteratorClass,
\]
and the resulting feasible label set
\[
\widehat C_{\mc}^\star(x;\alpha)
:=
\Bigl\{\iteratorClass\in[\numClass]:
p_\iteratorClass(x)\in
\bigcup_{j\in \widehat S_\iteratorClass^\star(\alpha_\iteratorClass)}
\widehat I_{\iteratorClass,j}
\Bigr\}.
\]

\clearpage
\section{Implementation Details}
\subsection{Baselines}
\label{app:baselines}
\begin{enumerate}
\item \textbf{\titleShort.} We construct \titleShorts from the calibration split. For Task~1, We adapt the procedure of Eq.~\eqref{eq:task-predicting_true_prob} for an $\epsilon$-expansion of \titleShorts.  For Task~2, we use Eq.~\eqref{eq:mc_acceptance}. That is, include class if its predicted probability lies in the an acceptance region by the subset solver in Algorithm~\ref{alg:mc_population_optimal}.

\item \textbf{Simultaneous confidence interval.}
Suppose a point estimator $p_{\theta}$ is parametrized by $\theta \in \mathbb{R}^d$. 
Given covariate value $x$, denote $p_{\theta^*}(x)$ as the true conditional probability when correctly specified. The simultaneous confidence interval for $p_{\theta^*}(x)$ is computed from a Wald confidence region for $\theta^*$ (using the chi-square critical value).
Let $\hat{\theta}$ be the fitted parameter and let $\widehat V$ be an estimate of its covariance matrix.
Let
\[
p_\theta(x) = g^{-1}\!\bigl(z(x)^\top \theta\bigr),
\]
where $z(x)$ is the feature vector (for example, $z(x) = (1~x^\top)^\top$) and $g^{-1}$ is the inverse link function. 
The simultaneous confidence interval for $p_{\theta^*}(x)$ is
\[
\mathrm{SCI}_{1-\alpha}(x)
=
\left[
\inf_{\theta \in \mathcal C_{1-\alpha}} p_\theta(x),
\;
\sup_{\theta \in \mathcal C_{1-\alpha}} p_\theta(x)
\right],
\]
where 
\[
\mathcal C_{1-\alpha}
=
\left\{
\theta \in \mathbb{R}^d :
(\theta-\hat\theta)^\top \widehat V^{-1}(\theta-\hat\theta)
\le \chi^2_{d,\,1-\alpha}
\right\}.
\]

\item \textbf{Split conformal prediction.} For each test covariate input $x$, the split conformal interval is 
\[
\widehat C_{1-\alpha}(x)
=
\bigl[\hat p(x)-\hat q_{1-\alpha},\ \hat p(x)+\hat q_{1-\alpha}\bigr]\cap[0,1],
\]
where $\hat q_{1-\alpha}$ is an empirical $1-\alpha$ quantile of the residual scores $|Y_i-\hat p(X_i)|$ on $n_\mathrm{cal}$ calibration data points. The coverage satisfies
\begin{align}\label{eq:conformal_tp}
\PP\Big(Y \in \widehat C_{1-\alpha}(x) \mid \{X_i, Y_i\}_{i\in [n_\mathrm{cal}]}\Big) \in \Big[1-\alpha,1-\alpha+(1+n_\mathrm{cal})^{-1}\Big)
\end{align}
\citep{papadopoulos2002inductive}.

\item \textbf{Calibration-based interval.}
We discretize $[0,1]$ into bins $B_k=(b_{k-1},b_k]$ with midpoint $m_k=(b_{k-1}+b_k)/2$. 
We retain a subset of bins such that their mass-weighted absolute calibration error is close to that of \titleShort; if that bin is not retained, we use the retained bin whose midpoint is closest to $\hat p(x)$.

\item \textbf{Fixed-width binning.} We discretize $[0,1]$ into fixed bins and use the bin containing $\hat p(x)$ as the reported interval.

\item \textbf{Conformal classification.}
This an additional baseline only for Task~2 (which differs from the others by not relying on intervals of point predictions). Let $\hat p_g(x)$ denote the predicted probability of class $g$. 
Given a test covariate input $x$, the label set includes classes whose predicted probability exceeds a threshold $\hat \tau_\alpha$,
\[
\widehat \Gamma_{1-\alpha}(x)
=
\{g \in [K] : \hat p_g(x) \ge \hat \tau_\alpha\},
\]
where $\hat \tau_\alpha$ is obtained from the lower tail of scores on calibration data:
\[
\hat \tau_\alpha = \lceil\alpha(n_\mathrm{cal} + 1)\rceil \text{ smallest of $\{\hat p_{Y_i}(X_i)\}_{ i\in{[n_\mathrm{cal}]}}$}.
\]
Analogous to Eq.~\eqref{eq:conformal_tp}, it is marginally guaranteed that on the test data 
\[
\PP\Big(Y \in \widehat \Gamma_{1-\alpha}(x)\mid \{X_i, Y_i\}_{i\in [n_\mathrm{cal}]}\Big) \in \Big[1-\alpha,1-\alpha+(1+n_\mathrm{cal})^{-1}\Big).
\]
\end{enumerate}

\paragraph{Interval subset solver for Task 2} For \titleShort, fixed-width binning, calibration-based intervals, and split conformal prediction, the classwise candidate intervals from the calibration set are passed to Algorithm~\ref{alg:mc_population_optimal}. For each class the solver keeps a min-mass subset whose plugin two-sided certificate meets the target \(\alpha\). Simultaneous confidence intervals and conformal classification bypass the solver. For simultaneous CIs we include class \(g\) if the upper endpoint satisfies \(U_g(x)\ge\tau\), with \(\tau\) the \(\lceil\alpha(n_{\mathrm{cal}}+1)\rceil\)-th smallest of \(\{U_{Y_i}(X_i)\}\). For conformal classification, we use the calibration set to select a threshold on point predictions.

\subsection{Metrics}
\label{app:metrics}
For Task~1, we report the average interval length, calibration errors, the coverage of the prediction space $[0,1]$, and the coverage of the target probability $p^*(x)$ on the test data. For a method that outputs intervals $\widehat C(X_i)=[L_i,U_i]$ on evaluation points $\{X_i\}_{i=1}^{n_{\mathrm{eval}}}$, let $\mathcal I=\{I_1,\dots,I_M\}$ denote the set of distinct reported intervals, where $I_j=[\ell_j,u_j]$. For each interval $I_j$, define its midpoint as $m_j=(\ell_j+u_j)/2$ and let $A_j=\{i:\widehat C(X_i)=I_j\}$ be the set of evaluation points assigned to that interval. We estimate the expected calibration error of the intervals relative to their midpoints by:
\[
\mathrm{ECE_{mid}}
=
\frac{1}{M}\sum_{j=1}^M
\left|
\frac{1}{|A_j|}\sum_{i\in A_j} p^*(X_i)-m_j
\right|.
\]
We also estimate the expected calibration error of the intervals relative to the predictions $p(X_i)$ by:
\[
\mathrm{ECE_{pred}}
=
\frac{1}{M}\sum_{j=1}^M
\left|
\frac{1}{|A_j|}\sum_{i\in A_j} p^*(X_i)- \frac{1}{|A_j|}\sum_{i\in A_j} p(X_i)
\right|.
\]
Intuitively, $\mathrm{ECE_{mid}}$ assesses the discrepancy between an interval’s midpoint and the average target probability $p^*(X)$ among points assigned to that interval, capturing how well the interval is centered around the target probabilities.
$\mathrm{ECE_{pred}}$ measures the discrepancy between the average prediction $p(X)$ and the average target probability $p^*(X)$ among points assigned to each interval, representing whether the grouping aligns with the underlying predictions.
The reported relative ECE is the mean ECE of each method divided by the smallest mean ECE across all methods.

For Task 2, we measure class-conditional error on label set $C(X;\tau)$,
\[
\frac{\sum_{i=1}^{n_{\text{eval}}} q_g(X_i)\mathbf{1}\{g\notin C(X_i;\tau)\}}
     {\sum_{i=1}^{n_{\text{eval}}} q_g(X_i)}.
\] We also report the average prediction-set size.

\subsection{Setup for Task 1}\label{app:task1_setup}

\paragraph{Univariate interval visualization (Figure~\ref{fig:task1_interval_visualization}).}
The covariate is sampled uniformly from a discrete grid $\mathrm{X\_GRID} = \{-2,\ldots,2\}$ with $1000$ equally spaced points, and the noise-free target probability used for evaluation is
\[
p^*(x) = \sigma(2x),
\]
where $\sigma(t) = 1/(1+e^{-t})$.
Observed labels are generated from the noisy logistic model
\[
Y \mid X=x,\xi \sim \mathrm{Bernoulli}\!\left(\sigma(2x + \xi)\right),
\qquad
\xi \sim N(0, 0.1^2).
\]
Thus the marginal conditional label probability is $\PP(Y=1\mid X=x)=\mathbb E_\xi[\sigma(2x+\xi)]$, whereas the reported target-probability coverage uses the noise-free $p^*(x)=\sigma(2x)$ above.
We set 
\[
n_{\mathrm{train}} = 2000,\qquad
n_{\mathrm{cal}} = 2000,\qquad
n_{\mathrm{eval}} = 10000.
\]
The common method settings are $\alpha = 0.1$, PICPI with \texttt{num\_bin}$=50$ in empirical mode, PICPI expansion parameter $\epsilon = 0.02$, and $50$ bins for the calibration-based and fixed-width baselines. The shared point predictor is scikit-learn logistic regression with its default $\ell_2$ penalty and regularization strength $1$.

\paragraph{Misspecified decision-tree illustration (Figure~\ref{fig:task1_misspecified_tree}).}
For the qualitative misspecification figure, the covariate is drawn from
\[
X \sim \mathrm{Unif}[-2,2],
\]
with true probability
\[
p^*(x) = \sigma\!\left(1.2\sin(2.5x) + 0.8x\right),
\]
and labels
\[
Y \mid X=x \sim \mathrm{Bernoulli}(p^*(x)).
\]
The point predictor is a deliberately misspecified decision tree with \texttt{max\_depth = 5} and \texttt{min\_samples\_leaf = 25}, fitted on
\[
n_{\mathrm{train}} = 20000,\qquad
n_{\mathrm{cal}} = 20000
\]
samples. PICPI is built with \texttt{num\_bin}$=100$ candidate grid points in empirical mode and expanded by $\epsilon=0.02$. The curves are plotted on a dense evaluation grid of $600$ evenly spaced points in $[-2,2]$.

\paragraph{Multivariate result table (Table~\ref{tab:task1_multivariate_mc_summary}).}
The covariate is
\[
X \sim N(0, I_{20}),
\]
and the true target probability is
\[
p^*(x) = \sigma(\beta_0 + x^\top \beta),
\qquad
\beta_0 = 0,
\qquad
\beta = (1.6,\,-1.1,\,0.9,\,0,\ldots,0)^\top \in \mathbb{R}^{20}.
\]
Labels are again generated from a noisy logistic model,
\[
Y \mid X=x,\xi \sim \mathrm{Bernoulli}\!\left(\sigma(\beta_0 + x^\top \beta + \xi)\right),
\qquad
\xi \sim N(0, 0.1^2).
\]
As in the univariate experiment, the reported target-probability coverage uses the noise-free $p^*(x)$ above, not the Gaussian-averaged marginal conditional label probability.
Each replication uses
\[
n_{\mathrm{train}} = 2000,\qquad
n_{\mathrm{cal}} = 2000,\qquad
n_{\mathrm{eval}} = 10000,
\]
and the table is averaged over $100$ Monte Carlo replications. The shared point predictor is scikit-learn logistic regression with its default $\ell_2$ penalty. Unless otherwise stated, the method hyperparameters are the same as in the univariate setting: $\alpha = 0.1$, PICPI with \texttt{num\_bin}$=50$ candidate grid points and $\varepsilon_{\mathrm{PICPI}} = 0.02$, and $50$ bins for the two binning baselines.

\paragraph{Split conformal implementation.}
For the Task~1 split conformal baseline, the implementation computes residual scores $R_i=|Y_i-\hat p(X_i)|$ on the calibration split and sets
\[
u=\min\!\left\{\frac{\lceil(1-\alpha)(n_{\mathrm{cal}}+1)\rceil}{n_{\mathrm{cal}}},1\right\},
\qquad
\hat q_{1-\alpha}=\operatorname{Quantile}^{\mathrm{linear}}_u
\bigl(\{R_i\}_{i\in[n_{\mathrm{cal}}]}\bigr).
\]
Here $\operatorname{Quantile}^{\mathrm{linear}}$ denotes the default linear-interpolation convention of \texttt{numpy.quantile}. This numerical implementation uses an interpolated empirical quantile rather than the exact conformal order statistic.

\paragraph{Empirical mode diagnostics (Figure~\ref{fig:emp_mode_diagnostics}).}
\label{app:task1_setup_emp_mode_diagnostics} 
The covariate and noisy-logit DGP, training size $n_{\mathrm{train}}=2000$, and fitted logistic model are the same as for Table~\ref{tab:task1_multivariate_mc_summary}. We compare empirical-mode and population-mode PICPI using $100$ equal-width score bins, with $\delta=0.1$ for population mode. We vary the calibration size over $\{10^3,3\times10^3,10^4,3\times10^4,10^5,3\times10^5,10^6,3\times10^6,10^7\}$, use $n_{\mathrm{eval}}=10^6$, and run $50$ replications for each calibration size. In these diagnostics, the held-out true probability is the Gaussian-averaged marginal probability $\mathbb E_\xi[\sigma(\beta_0+x^\top\beta+\xi)]$.

\subsection{Task 1: Additional Result}
\begin{figure}[ht]
\centering
\includegraphics[width=\linewidth]{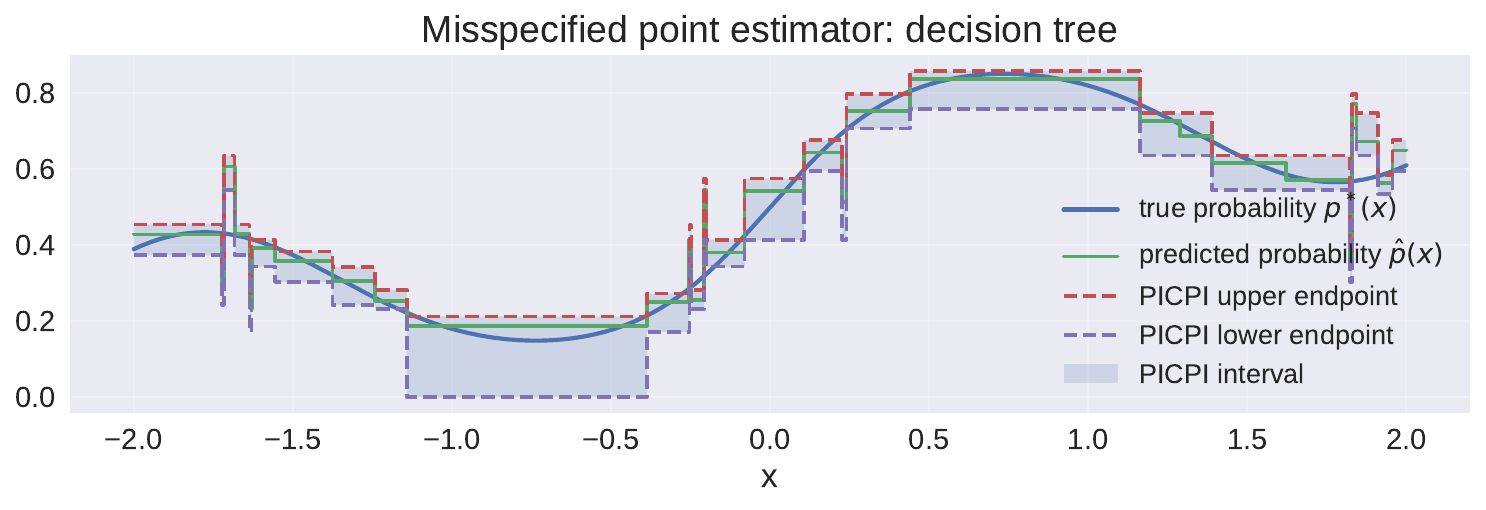}
\caption{Task~1 demonstration with a misspecified decision-tree point estimator. \titleShorts adapt to biased, piecewise-constant predictions by adjusting interval endpoints and widening where calibration reveals instability or mismatch. The smooth curve shows the true conditional probability $p^*(x)$, while the step function shows the misspecified decision-tree estimate $\hat p(x)$. The dashed step curves mark the lower and upper endpoints of the \titleShort interval, and the shaded boxes show the resulting interval assigned at each $x$.}
\label{fig:task1_misspecified_tree}
\end{figure}

In the misspecified decision-tree example of Figure~\ref{fig:task1_misspecified_tree}, the \titleShorts remain responsive when the fitted predictor is structurally biased: the decision tree produces a piecewise-constant \(\hat p(x)\) that fails to track the nonlinear true probability curve, but \titleShort adapts its interval endpoints around these score plateaus and widens in regions where the calibration data reveal instability or mismatch.

\subsection{Setup for Task 2}\label{app:task2_setup}

\paragraph{Coverage sweep for two DGPs (Figure~\ref{fig:task2})}
 Both use a $d = 10$ dimensional covariate and $\numClass = 4$ classes. Given the true class-probability vector $q(x) \in [0,1]^{\numClass}$, labels are drawn as
\[
Y \mid X = x \sim \mathrm{Categorical}\big(q(x)\big).
\]
Each replication uses disjoint splits of sizes
\[
n_{\mathrm{train}} = n_{\mathrm{cal}} = n_{\mathrm{eval}} = 2000,
\]
and all curves are averages over $1000$ independent replications. \titleShort uses \texttt{num\_bin}$=100$ equally spaced candidate grid points in empirical mode, while the calibration-based and fixed-width baselines use $100$ bins. Candidate intervals from \titleShort, split conformal prediction, calibration-based intervals, and fixed-width binning are passed to the classwise interval-selection solver (Algorithm~\ref{alg:mc_population_optimal}). For split conformal prediction, each class uses at most $200$ interval centers sampled without replacement from its calibration scores. Simultaneous confidence intervals and conformal classification instead use a calibration threshold, as described in Section~\ref{sec:experiment}. The coverage sweep evaluates all methods at $16$ equally spaced target coverage levels $1 - \alpha \in \{0, 1/15, \ldots, 1\}$.

\paragraph{DGP 1: Linear-Gaussian.}
The covariate is drawn from
\[
X \sim N(0, I_{10}),
\]
and the true class probabilities follow a linear softmax model,
\[
q(x) = \mathrm{softmax}(W x + b),
\]
with a sparse weight matrix
{\small
\[
\setlength{\arraycolsep}{3pt}
W =
\begin{bmatrix}
1.4 & -1.0 & 0.8 & 0 & 0 & 0 & 0 & 0 & 0 & 0 \\
0 & -0.7 & 0 & 1.1 & 0.8 & 0 & 0 & 0 & 0 & 0 \\
0 & 0 & 0.9 & 0 & 0 & -1.2 & 1.0 & 0 & 0 & 0 \\
-0.6 & 0 & 0 & 0 & 0 & 0 & 0 & 0.8 & -0.9 & 1.2
\end{bmatrix}
\]
}%
and an intercept $b = (0.25,\,-0.20,\,0.15,\,-0.20)^\top$.
The multinomial logistic working model is therefore correctly specified.

\paragraph{DGP 2: Nonlinear-mixture.}
The covariate is drawn from a four-component latent-factor mixture with component weights
\[
(0.32,\, 0.24,\, 0.27,\, 0.17).
\]
Equivalently, a latent component label is first drawn as
\[
k \sim \mathrm{Categorical}(0.32,\, 0.24,\, 0.27,\, 0.17).
\]
Conditional on component $k\in\{1,\ldots, 4\}$, an initial draw $\widetilde{X}\in\mathbb{R}^{10}$ is generated as
\[
\widetilde{X} = \mu_k + \Lambda^\top Z + \mathrm{diag}(\sigma_k)\, h(Z)\, S\, \varepsilon,
\qquad
Z \sim N(0, I_2),
\quad
\varepsilon \sim N(0, I_{10}),
\]
where $S = \sqrt{7/\chi^2_7}$ induces Student-$t$ (heavy) tails and $h(z) = 1 + 0.25\tanh(z_1 + 0.45 z_2)$ induces heteroscedasticity in the latent factors. The component-specific location and scale vectors $\mu_k, \sigma_k \in \mathbb{R}^{10}$ are the rows of
{\small
\begin{gather*}
\setlength{\arraycolsep}{3pt}
M =
\begin{bmatrix}
-1.4 & 0.5 & -0.9 & 0.2 & -0.3 & 0.6 & -0.5 & 0.4 & -0.7 & 0.3 \\
-0.2 & -1.0 & 0.7 & -0.6 & 0.8 & -0.4 & 0.3 & -0.5 & 0.6 & -0.8 \\
0.9 & 0.6 & -0.2 & 1.1 & -0.7 & 0.4 & 0.9 & 0.7 & -0.1 & 0.5 \\
1.6 & -0.3 & 1.0 & 0.5 & 0.4 & -0.8 & -1.0 & -0.6 & 0.9 & -0.4
\end{bmatrix},
\\
\setlength{\arraycolsep}{3pt}
V =
\begin{bmatrix}
0.60 & 0.45 & 0.55 & 0.40 & 0.50 & 0.35 & 0.45 & 0.42 & 0.50 & 0.38 \\
0.45 & 0.65 & 0.50 & 0.60 & 0.55 & 0.40 & 0.35 & 0.48 & 0.44 & 0.57 \\
0.55 & 0.50 & 0.70 & 0.45 & 0.60 & 0.50 & 0.55 & 0.52 & 0.46 & 0.50 \\
0.50 & 0.40 & 0.60 & 0.55 & 0.45 & 0.65 & 0.60 & 0.58 & 0.54 & 0.47
\end{bmatrix},
\end{gather*}
}%
and the shared factor-loading matrix is
{\small
\[
\setlength{\arraycolsep}{3pt}
\Lambda =
\begin{bmatrix}
0.95 & -0.55 & 0.35 & 0.45 & -0.30 & 0.40 & -0.20 & 0.32 & -0.28 & 0.24 \\
0.30 & 0.60 & -0.45 & 0.20 & 0.55 & -0.35 & 0.65 & -0.25 & 0.48 & -0.42
\end{bmatrix}.
\]
}%
The covariate $X$ is then obtained from $\widetilde{X}$ by fixed deterministic warps. Starting from $x = \tilde{x} + 0.22\sin(\tilde{x}_1)\,\delta_k$, the coordinates are updated sequentially (each line using the current values) via
\begin{align*}
x_3 &\leftarrow x_3 + 0.28\,(x_1^2 - 1), &
x_5 &\leftarrow x_5 + 0.16\, x_2 x_4, &
x_7 &\leftarrow x_7 + 0.14\, x_3 x_6, \\
x_8 &\leftarrow x_8 + 0.18\sin(x_1 x_3), &
x_9 &\leftarrow x_9 + 0.15\, x_5 x_7, &
x_{10} &\leftarrow x_{10} + 0.12\,(x_8^2 - 1),
\end{align*}
followed by the component offset $x \leftarrow x + o_k$ and the final warps
\begin{align*}
x_6 &\leftarrow x_6 + 0.20\,\mathrm{sgn}(x_1)\sqrt{|x_1| + 10^{-6}}, &
x_2 &\leftarrow x_2 - 0.15\tanh(x_7 - x_5), \\
x_9 &\leftarrow x_9 + 0.18\,\mathrm{sgn}(x_8)\sqrt{|x_8| + 10^{-6}}, &
x_{10} &\leftarrow x_{10} - 0.12\tanh(x_9 - x_6).
\end{align*}
The drift and offset vectors $\delta_k, o_k \in \mathbb{R}^{10}$ are the rows of
{\small
\begin{gather*}
\setlength{\arraycolsep}{3pt}
D =
\begin{bmatrix}
0.30 & -0.20 & 0.10 & 0.00 & 0.12 & -0.08 & 0.05 & 0.07 & -0.06 & 0.04 \\
-0.10 & 0.28 & -0.12 & 0.08 & -0.06 & 0.14 & -0.04 & -0.09 & 0.10 & -0.05 \\
0.08 & -0.12 & 0.26 & -0.16 & 0.10 & -0.06 & 0.18 & 0.11 & -0.04 & 0.09 \\
-0.14 & 0.06 & -0.08 & 0.24 & -0.10 & 0.20 & -0.22 & -0.05 & 0.12 & -0.13
\end{bmatrix},
\\
\setlength{\arraycolsep}{3pt}
O =
\begin{bmatrix}
-0.12 & 0.06 & -0.04 & 0.02 & -0.05 & 0.03 & -0.02 & 0.04 & -0.03 & 0.01 \\
0.05 & -0.10 & 0.08 & -0.06 & 0.09 & -0.04 & 0.03 & -0.02 & 0.05 & -0.07 \\
0.10 & 0.04 & -0.06 & 0.11 & -0.08 & 0.06 & 0.12 & 0.03 & -0.09 & 0.08 \\
0.16 & -0.02 & 0.10 & 0.07 & 0.05 & -0.11 & -0.14 & -0.08 & 0.07 & -0.05
\end{bmatrix}.
\end{gather*}
}%
The resulting marginal covariate law is non-Gaussian, multimodal, and dependent across coordinates.
The true class probabilities follow a softmax model whose linear logits are rescaled by a covariate-dependent scale function,
\begin{gather*}
q(x) = \mathrm{softmax}\!\left(\frac{W x + b}{s(x)}\right),
\\
s(x) = \mathrm{clip}\big(0.86 + 0.22\tanh(0.60 x_1 - 0.35 x_4 + 0.20 x_9),\ 0.65,\ 1.40\big),
\end{gather*}
where the weight matrix $W \in \mathbb{R}^{4 \times 10}$ and intercept $b \in \mathbb{R}^4$ were generated once (by combining i.i.d.\ Gaussian entries with structured rank-one terms drawn from a fixed seed) and held constant across all replications; rounded to two decimals,
{\small
\[
\setlength{\arraycolsep}{3pt}
W =
\begin{bmatrix}
-1.49 & 1.57 & -0.40 & -1.12 & 1.37 & -0.67 & 0.81 & -0.92 & 0.96 & 0.07 \\
-0.10 & 0.32 & -0.75 & -0.45 & 0.29 & -0.14 & 0.54 & -0.37 & -0.30 & -0.10 \\
0.32 & -0.65 & 0.47 & -0.05 & 0.11 & -0.12 & -0.44 & 0.32 & 0.16 & 0.57 \\
1.53 & -1.48 & 1.06 & 1.23 & 0.05 & 1.37 & -0.50 & 0.91 & 0.18 & 0.41
\end{bmatrix}
\]
}%
and $b = (0.30,\, 0.16,\, -0.05,\, -0.41)^\top$.
Because the log-odds are nonlinear in $x$ through $s(x)$, the multinomial logistic working model is misspecified and thus is not expected to be well calibrated under this DGP.

\subsection{Setup for Figure~\ref{fig:thm52_box_rate}}

\label{app:verification-consistency-prediction-setup}

The covariate and noisy-label DGP are the same as in the univariate interval visualization of Section~\ref{app:task1_setup}. We use Algorithm~\ref{alg:calibration_population} with \(100\) bins and \(\delta=0.1\) and \(100\) repetitions.
The noise-free target is $p^*(x)=\sigma(2x)$. We fit an unpenalized logistic-regression estimator once using a training set of size $10,000$, fix an evaluation set of size $10,000$, and vary the calibration-set size over $\{10^3,3\times10^3,10^4,3\times10^4,10^5,3\times10^5,10^6,3\times10^6,10^7\}$. 
\begin{enumerate}
\item fit a logistic-regression estimator \(\hat p\) on the training fold;
\item apply the calibration procedure on the calibration fold to obtain an interval collection \(\mathcal I=\{[a_j,b_j]\}_j\);
\item evaluate the interval system on the independent evaluation fold by computing, for each point, the minimal width among all intervals that contain the target probability \(p^*(X_i)\), with width \(1.0\) assigned if no interval contains \(p^*(X_i)\).
\end{enumerate}

\section{Additional Algorithms}\label{app:alg}

Algorithm~\ref{alg:max_coverage_intervals} selects a disjoint subset for intervals returned by Algorithm~\ref{alg:calibration_population}. The \textsc{DisjointCalibration} procedure first imposes a maximum length and
then uses weighted interval scheduling to select a non-overlapping subcollection with maximum covered length. 

The \textsc{MaxCoverageIntervals} procedure solves a weighted interval scheduling problem to select a set of non-overlapping intervals that maximizes total coverage, where each interval is weighted by its length. It sorts intervals by end time, precomputes the last compatible interval for each candidate, and then uses dynamic programming to recursively compute the optimal value, followed by a backtracking step to recover the selected intervals.

\begin{algorithm*}[h]
   \caption{Maximum-coverage non-overlapping intervals (weighted interval scheduling)}
   \label{alg:max_coverage_intervals}
   \small
\begin{algorithmic}[1]
\Procedure{\textsc{DisjointCalibration}}{length bound $l$, \titleShorts $\intervalset_p$}
\State
\Comment{Restrict \titleShorts to intervals of length $\le l$ and maximize covered length}
\State $\intervalset_p^{(\le l)} \la \emptyset$
\For{{\bf each} interval $[a,b] \in \intervalset_p$}
    \If{$b - a \le l$}
        \State $\intervalset_p^{(\le l)} \la \intervalset_p^{(\le l)} \cup \{[a,b]\}$
    \EndIf
\EndFor
\If{$\intervalset_p^{(\le l)} = \emptyset$}
    \State {\bf Return} $\emptyset$ \Comment{No interval satisfies the length constraint}
\EndIf
\State Relabel $\intervalset_p^{(\le l)}$ as $\{[s_i, e_i]\}_{i=1}^m$
\State $\mathcal{S}^\star \la \textsc{MaxCoverageIntervals}\big(\{[s_i, e_i]\}_{i=1}^m\big)$

\Comment{Select a disjoint set intervals of maximal coverage}
\State {\bf Return} $\mathcal{S}^\star$
\EndProcedure
\\
\Procedure{\textsc{MaxCoverageIntervals}}{Intervals $\{[s_i, e_i]\}_{i=1}^m$}

\Comment{Find a maximum-length set of non-overlapping intervals}
\State Sort the intervals so that $e_1 \le e_2 \le \cdots \le e_m$
\For{$j \la 1$ {\bf to} $m$}
    \State $w_j \la e_j - s_j$ \Comment{Weight = interval length}
\EndFor
\For{$j \la 1$ {\bf to} $m$}
    \State $p(j) \la \max\{\, i < j : e_i \le s_j \,\}$, or $0$ if no such $i$ exists
    \Comment{Binary search over $\{e_1,\dots,e_{j-1}\}$}
\EndFor
\State Initialize array $\text{OPT}[0..m]$ \Comment{Dynamic Programming}
\State $\text{OPT}[0] \la 0$
\For{$j \la 1$ {\bf to} $m$}
    \State $\text{OPT}[j] \la \max\bigl(w_j + \text{OPT}[p(j)],~\text{OPT}[j-1]\bigr)$\Comment{Maximum coverage using $\{1,2,\dots,j\}$}
\EndFor
\State Initialize selected set $\mathcal{S}^\star \la \emptyset$
\State $j \la m$
\While{$j > 0$}
    \If{$w_j + \text{OPT}[p(j)] > \text{OPT}[j-1]$}
        \State $\mathcal{S}^\star \la \mathcal{S}^\star \cup \{[s_j, e_j]\}$
        \State $j \la p(j)$
    \Else
        \State $j \la j - 1$
    \EndIf
\EndWhile
\State {\bf Return} $\mathcal{S}^\star$
\EndProcedure

\end{algorithmic}
\end{algorithm*}

\end{document}